%% file: main.tex
\documentclass[pdflatex,sn-mathphys-num]{sn-jnl}

\usepackage{graphicx}
\usepackage{multirow}
\usepackage{amsmath,amssymb,amsfonts,amsthm}
\usepackage[title]{appendix}
\usepackage{xcolor}
\usepackage{textcomp}
\usepackage{manyfoot}
\usepackage{booktabs}
\usepackage{listings}

\usepackage{etoolbox,xspace}
\usepackage{thmtools,thm-restate}
\usepackage[sans]{dsfont}
\makeatletter
\@ifundefined{newcounteralias}{}{%
  \renewcommand\thmt@autorefsetup{%
    \@xa\def\csname\thmt@envname autorefname\@xa\endcsname
      \@xa{\thmt@thmname}%
  }%
}
\makeatother
\theoremstyle{thmstyleone}
\newtheorem{theorem}{Theorem}
\newtheorem{lemma}[theorem]{Lemma}

\newtheorem{corollary}[theorem]{Corollary}

\theoremstyle{thmstyletwo}

\newtheorem{remark}{Remark}

\theoremstyle{thmstylethree}
\newtheorem{definition}{Definition}

\def\tsc#1{\csdef{#1}{\textsc{\lowercase{#1}}\xspace}}
\tsc{WGM}
\tsc{QE}

\newcommand{\Var}{\text{Var}}

\newcommand{\x}{\mathbf{x}}
\renewcommand{\Pr}{\text{Pr}}
\newcommand{\R}{\mathbb{R}}

\begin{document}

% ---------- Full title only ----------
\title{Accurate Trace Estimation with Fewer Random Bits via Recursive TensorSketch}

% ---------- Authors ----------
\author*[1]{%
  \fnm{Mohammad Azhar} \sur{Khan}}
\email{cs24mtech12006@iith.ac.in}

\author[1]{%
  \fnm{Rameshwar} \sur{Pratap}~%
  }
\email{rameshwar@cse.iith.ac.in}

\author[1]{%
  \fnm{Amit} \sur{Sharma}~%
}
\email{cs24resch02002@iith.ac.in}

% ---------- Shared affiliation ----------
\affil[1]{%
  \orgname{Indian Institute of Technology Hyderabad},
  \orgaddress{%
    \street{Kandi},
    \city{Sangareddy},
    \postcode{502284},
    \state{Telangana},
    \country{India}%
  }%
}

% ---------- Abstract ----------
\abstract{%
We consider the problem of estimating the trace of an implicit matrix
$\mathbf{A} \in \R^{d^p\times d^p}$ that can only be accessed through
matrix-vector products queries. The \textit{Hutchinson trace estimator}%
~\cite{Girard1987algorithme, article-hutchinson} is a classical sketching
method for this problem. Their estimator,
$H_{m}(\mathbf{A}) = \frac{1}{m}
\sum_{i=1}^{m} {\mathbf{z}^{(i)}}^T \mathbf{A} \mathbf{z}^{(i)}, \quad
\text{where } \ {\mathbf{z}^{(i)}}\in \mathbb{R}^{d^p}$,
and $z^{(i)}_j \in \mathcal{N}(0, 1), j\in [d^p]$, satisfies the following
guarantees: (i)
$\mathbb{E}[H_{m}(\mathbf{A})]=\operatorname{tr}(\mathbf{A})$, and (ii)
$\mathrm{Var}[H_{m}(\mathbf{A})]=\frac{2}{m}||\mathbf{A}||_F^2$.
Generating one query vector $\mathbf{z}^{(i)}$ requires $O(d^p)$ random
bits; thus, $m$ queries require $O(md^p)$ random bits, which can be
prohibitive in large-scale applications. Recent work by Meyer et
al.~\cite{meyer2025hutchinsonsestimatorbadkroneckertraceestimation}
proposes a variant of the Hutchinson trace estimator in which each query
vector in $\mathbb{R}^{d^p}$ is constructed as the Kronecker product of
$p$ random vectors in $\mathbb{R}^d$, requiring $O(mpd)$ random bits for
$m$ query vectors. The estimator
of~\cite{meyer2025hutchinsonsestimatorbadkroneckertraceestimation} is
unbiased; however, its variance grows exponentially with $p$. In this
work, we address this limitation by proposing a sketching-based estimator
that requires $O\!\big(p (d + m)\log m\big)$ random bits, yields an
unbiased estimate of the trace, and simultaneously achieves a variance
bound that grows polynomially with $p$.
}

% ---------- Keywords ----------

\maketitle

% ---------- Main manuscript ----------
% These files contain section text only, without another preamble
% or document environment.
%
% For the template's single-file submission requirement, replace
% these input commands with the corresponding section contents.

\input{introduction}

\input{related_work}

\input{background}
\input{real_rts_analysis}
\input{complex_rts_analysis}
\input{conclusion}

% ---------- Appendix ----------
\appendix
\input{appendix}

% ---------- Bibliography ----------
\bibliography{reference}

\end{document}

%% file: introduction.tex
\section{Introduction}

A central problem in scientific computing is the estimation of the trace of a large matrix $\mathbf{A} \in \mathbb{R}^{d^p\times d^p}$ when explicit access to its entries is restricted. Instead, the matrix is accessible only through an oracle that returns matrix–vector products $\mathbf{A}\mathbf{x}$ for arbitrary vectors $\mathbf{x} \in \mathbb{R}^{d^p}$. Under this restricted access model, the goal is to develop efficient algorithms that approximate the \textit{trace} of matrix $\mathbf{A}$ while minimizing the number of oracle queries.  The \textit{matrix–vector} oracle model, also referred to as the \textit{implicit matrix} model, is a widely adopted computational framework in the numerical linear algebra community~\cite{sun2021querying, chen2023krylov, DBLP:journals/nla/HalikiasT24,DBLP:conf/stoc/BakshiCW22,DBLP:journals/siammax/PerssonCK22,DBLP:conf/sosa/MeyerMMW21}. The trace estimation problem in \textit{implicit matrix} model can be solved exactly using $D=d^p$ oracle queries by using the standard basis vectors $\mathbf{e}_1, \mathbf{e}_2, \ldots, \mathbf{e}_D$ via the following estimator $\operatorname{tr}(\mathbf{A})=\sum_{i=1}^D \mathbf{e}_i^T \mathbf{A}\mathbf{e}_i$. Each term in the summation corresponds to a single diagonal entry of $\mathbf{A}$, leading to a total of $O(d^p)$ matrix–vector queries. However, the computational cost associated with such a large number of oracle queries is prohibitive. 

\noindent The seminal algorithm due to Girard and Hutchinson~\cite{Girard1987algorithme, article-hutchinson}, known as the \textit{Hutchinson trace estimator}, provides an efficient approximation for trace estimation. 
Given an implicit matrix $\mathbf{A} \in \mathbb{R}^{d^p \times d^p}$, the Hutchinson trace estimator is  defined as, 
$
    H(\mathbf{A})=  {\mathbf{z}}^T \mathbf{A} \mathbf{z},
$
where $\mathbf{z} \in \mathbb{R}^{d^p}$ with i.i.d.\ entries  ${z}_i\sim \mathcal{N}(0, 1)$, for $i\in [d^p]$. The estimator satisfies the following guarantee 
$
\mathbb{E}[H(\mathbf{A})]=\operatorname{tr}(\mathbf{A}), \quad \mathrm{Var}[H(\mathbf{A})]=2||\mathbf{A}||_F^2.
$ Furthermore, to reduce the variance, the above procedure is repeated independently 
$m$ times, and the final estimator is defined as the mean of these 
$m$ estimators, that is, 
\begin{align}\label{huts_estimator_eq}
    H_{m}(\mathbf{A}) = \frac{1}{m} \sum_{i=1}^{m} {\mathbf{z}^{(i)}}^T \mathbf{A} \mathbf{z}^{(i)}, \quad \text{where } \ {\mathbf{z}^{(i)}}\in \mathbb{R}^{d^p}, z^{(i)}_j \in \mathcal{N}(0, 1), \text{~and~}j\in [d^p].
\end{align}

\noindent The estimator satisfies the following guarantee 
\begin{align}\label{huts_estimator_eq_guarantees}
    \mathbb{E}[H_{m}(\mathbf{A})]=\operatorname{tr}(\mathbf{A}), \quad \mathrm{Var}[H_{m}(\mathbf{A})]=\frac{2}{m}||\mathbf{A}||_F^2.
\end{align}

Subsequent work further improved the sample-complexity analysis of classical trace estimators.~\cite{roosta2015improved} derived sharper bounds for Gaussian, Rademacher, and unit-vector estimators, including a Hutchinson bound without the rank-dependent term appearing in the earlier analysis. Under the quadratic-form query model,~\cite{wimmer2014optimal} characterized optimal linear nonadaptive estimators and established lower bounds for multiplicative trace approximation. More recently, ~\cite{jiang2021optimal} studied nearly optimal high-probability trace-estimation sketches under matrix-vector access. These works primarily seek to reduce the number of oracle queries, whereas our work studies the complementary objective of reducing the randomness required to construct queries in kronecker-structured spaces.

 The Hutchinson trace estimator, as stated in Equation~\eqref{huts_estimator_eq}, requires $m$ random vectors $\mathbf{z}^{(i)} \in \mathbb{R}^{d^p}$. Consequently, the total number of random bits required by the estimator is $O(md^p)$. Structured random queries based on Kronecker-structured random vectors for trace estimation were proposed by~\cite{bujanovic2021norm}.
 Building on this idea,~\cite{meyer2025hutchinsonsestimatorbadkroneckertraceestimation} addresses the challenge of random bits and suggests an estimator that requires significantly fewer random bits. Their estimator construct a query vector $\x \in \mathbb{R}^{d^p}$ as a Kronecker product of $p$ independent random vectors in $\mathbb{R}^d$, that is, $\mathbf{x} = \mathbf{x}_1 \otimes \cdots \otimes \mathbf{x}_p,$ where  $\mathbf{x}_i \in \mathbb{R}^d$ for $i\in [p]$.
Therefore, generating a single random vector $\mathbf{x}$ requires $O(dp)$ random bits, and the final estimator - formed by averaging $m$ such estimators - requires $O(mdp)$ random bits, in contrast to the $O(md^p)$ random bits required by the Hutchinson trace estimator~\cite{article-hutchinson}.
However, the main limitation of their approach is that the variance of their estimator grows exponentially with $p$, that is, $O\left(\frac{3^p}{m} \left(\operatorname{tr}\left(\mathbf{A}\right)\right)^2\right)$ - making the estimator less accurate. This motivates the problem considered in this paper, which we state as follows:

\textit{\textbf{Problem Statement}: Given an implicit matrix $\mathbf{A} \in \mathbb{R}^{d^p \times d^p}$, the goal is to design a trace estimation algorithm that requires asymptotically fewer random bits and simultaneously provides an accurate trace estimation.}

We draw inspiration from recent advances in sketching techniques to address this problem. In particular,~\cite{doi:10.1137/1.9781611975994.9} introduced \texttt{Recursive TensorSketch}, an efficient sketching method for approximating high-degree polynomial kernels. Their approach enables effective compression of polynomial kernels using a sketching dimension that scales only polynomially with the degree of the kernel function. We leverage \texttt{Recursive TensorSketch} to design a trace estimator for an implicit matrix $\mathbf{A} \in \mathbb{R}^{d^p \times d^p}$. We prove that the proposed estimator is unbiased, requires asymptotically fewer random bits than the classical Hutchinson trace estimator~\cite{article-hutchinson}, and admits variance bounds with only polynomial dependence on the  $p$. This constitutes an exponential improvement in the dependence on $p$ over the variance bounds established for Kronecker-Hutchinson estimators in~\cite{meyer2025hutchinsonsestimatorbadkroneckertraceestimation}. We summarize our key contributions as follows:\\

\noindent \textbf{Our Contribution:}
\begin{itemize}
    \item We propose a novel trace estimator for an implicit matrix $\mathbf{A} \in \mathbb{R}^{d^p \times d^p}$. Our estimator, $\operatorname{tr}\big(\Pi^{p} \mathbf{A} (\Pi^{p})^\top\big)$ (see Definition~\ref{def:rts_trace_estimator}), leverages the \texttt{Recursive TensorSketch} matrix $\Pi^{p} \in \mathbb{R}^{m \times d^p}$ proposed by~\cite{doi:10.1137/1.9781611975994.9}.
    \item We show that the proposed estimator is unbiased and derive a variance bound that is a polynomial of degree $ p$. Further, when the input matrix $\mathbf{A}$ is a Positive Semi-Definite (PSD) matrix then the variance of our estimator achieves exponential improvement over the estimator proposed in~\cite{meyer2025hutchinsonsestimatorbadkroneckertraceestimation}. Furthermore, the number of random bits required by our estimator is $O\!\big(p (d + m)\log m\big)$, which is asymptotically better to that of required in~\cite{meyer2025hutchinsonsestimatorbadkroneckertraceestimation}. Also, it is exponentially smaller than the Hutchinson trace estimator ~\cite{article-hutchinson}, which requires $O(md^p)$ random bits.
    \item We further propose a complex-valued analogue of our estimator (Definition~\ref{def:complex_rts_trace_estimator}), in which the entries of the \texttt{Recursive TensorSketch} matrix $\Pi^{p} \in \mathbb{C}^{m \times d^p}$ are sampled from complex random variables. We show that this variant achieves variance that is exponentially smaller than that of~\cite{meyer2025hutchinsonsestimatorbadkroneckertraceestimation} conditioned that the input matrix is a PSD matrix, while simultaneously requiring asymptotically fewer random bits.
\end{itemize}

There are two complementary approaches to reducing the amount of randomness required by randomized sketching algorithms. One approach is to redesign the sketching construction so that its random choices are shared through an underlying structure. This is the approach pursued in this paper through \texttt{Recursive TensorSketch}. A different approach is to retain an existing sketching construction while reducing its randomness by implementing the underlying hash functions using tabulation-based hashing~\cite{DBLP:conf/stoc/CarterW77}. Prior work has shown that simple and double tabulation hashing can provide strong concentration guarantees for a variety of randomized algorithms and data structures, including Minwise Independent Permutations and Cuckoo Hashing~\cite{patrascu2011power,thorup2013simple}. However, hash functions generated by tabulation hashing are generally not $4$-wise independent and, therefore, cannot be directly used in standard trace estimation algorithms~\cite{Girard1987algorithme, article-hutchinson}, where such independence is required by the analysis. Nevertheless, tabulation hashing may offer an alternative approach for reducing the random seed required by existing trace-estimation sketches. An interesting direction for future work is to investigate whether such implementations can preserve the required JL moment properties and variance guarantees.

\noindent 
Trace estimation is a fundamental primitive with numerous large-scale applications. 
Hutchinson trace estimator~\cite{article-hutchinson} has been used extensively as a key subroutine in various applications such as sublinear-time spectral density estimation~\cite{10.1145/3519935.3520009}, 
faster eigenvalue approximation~\cite{10.1145/3564246.3585102},  counting triangles in large graphs~\cite{10.1109/ICDM.2008.72,avron2010counting},  approximating spectral sums \cite{DBLP:journals/siamsc/HanMAS17}, and estimating    $||\mathbf{A}||_F$  (using the well-known identity $||\mathbf{A}||_F^2= \operatorname{tr} \left(\mathbf{A}^T\mathbf{A} \right))$ to name a few. Our proposed estimator can be plugged into these applications in place of~\cite{article-hutchinson} to yield a randomness-efficient algorithm with almost the same accuracy.

%% file: related_work.tex
\section{Related Work}

\begin{table}[h]
    \centering
    \scriptsize
    \renewcommand{\arraystretch}{1.3}
    \setlength{\tabcolsep}{3pt}
    \resizebox{\textwidth}{!}{%
    \begin{tabular}{
        |p{3.5cm}
        |p{5.0cm}
        |p{3.8cm}
        |p{5.2cm}|
    }
        \hline
        \textbf{Estimator}
        & \textbf{Variance}
        & \textbf{Randomness}
        & \textbf{Bound on \# samples for an
          $(\varepsilon,\delta)$-approx.} \\
        \hline

        Hutchinson (Gaussian)
        \cite{10.1145/1944345.1944349}
        & $\displaystyle
          =\frac{2}{m}\|\mathbf A\|_F^2$
        & $O(md^p)$
        & $\displaystyle
          20\varepsilon^{-2}
          \ln\!\left(\frac{2}{\delta}\right)$ \\
        \hline

        Hutchinson (Rademacher)
        \cite{10.1145/1944345.1944349}
        & $\displaystyle
          =\frac{2}{m}
          \left(
              \|\mathbf A\|_F^2
              -
              \sum_{i=1}^{n}A_{ii}^2
          \right)$
        & $O(md^p)$
        & $\displaystyle
          6\varepsilon^{-2}
          \ln\!\left(\frac{2r}{\delta}\right)$ \\
        \hline

        Normalized Rayleigh quotient
        \cite{10.1145/1944345.1944349}
        & $= \frac{d^p}{m} \left(\sum_i^{d^p} A_{ii}^2 - (\operatorname{tr}(\mathbf{A}))^2 \right)$ 
        & $O(md^p)$
        & $\displaystyle
          \frac{n^2\kappa_f^2(\mathbf A)}
               {2r^2\varepsilon^2}
          \ln\!\left(\frac{2}{\delta}\right)$ \\
        \hline

        Unit vector estimator
        \cite{10.1145/1944345.1944349}
        & $= \frac{d^p}{m} \left(\sum_i^{d^p} A_{ii}^2 - (\operatorname{tr}(\mathbf{A}))^2 \right)$
        & $O(mp\log d)$
        & $\displaystyle
          \frac{r_D^2(\mathbf A)}
               {2\varepsilon^2}
          \ln\!\left(\frac{2}{\delta}\right)$ \\
        \hline

        Mixed unit vector estimator
        \cite{10.1145/1944345.1944349}
        & --
        & $O(mp\log d)$
        & $\displaystyle
          8\varepsilon^{-2}
          \ln\!\left(\frac{4n^2}{\delta}\right)
          \ln\!\left(\frac{4}{\delta}\right)$ \\
        \hline

        Kronecker-Hutchinson (real)
        \cite{meyer2025hutchinsonsestimatorbadkroneckertraceestimation}
        & $\displaystyle
          \le
          \frac{3^p}{m}
          \bigl(\operatorname{tr}(\mathbf A)\bigr)^2$
        & $O(mpd)$
        & $\displaystyle
              \frac{3^p}{\varepsilon^2}
              \ln\!\frac{1}{\delta}$ \\
        \hline

        Kronecker-Hutchinson (complex)
        \cite{meyer2025hutchinsonsestimatorbadkroneckertraceestimation}
        & $\displaystyle
          \le
          \frac{2^p}{m}
          \bigl(\operatorname{tr}(\mathbf A)\bigr)^2$
        & $O(mpd)$
        & $\displaystyle
              \frac{2^p}{\varepsilon^2}
              \ln\!\frac{1}{\delta}$ \\
        \hline

        Recursive TensorSketch (real) [this paper]
        & $\displaystyle
          \le
          \left(
              \frac{10p}{m}
              +
              \frac{100p^2}{m^2}
          \right)
          \bigl(\operatorname{tr}(\mathbf A)\bigr)^2$
        & $O\!\bigl(p(d+m)\log m\bigr)$
        & $\displaystyle
          \frac{20p}{\varepsilon^2\delta}$ \\
        \hline

        Recursive TensorSketch (complex) [this paper]
        & $\displaystyle
          \le
          \left(
              \frac{4p}{m}
              +
              \frac{16p^2}{m^2}
          \right)
          \bigl(\operatorname{tr}(\mathbf A)\bigr)^2$
        & $O\!\bigl(p(d+m)\log m\bigr)$
        & $\displaystyle
          \frac{8p}{\varepsilon^2\delta}$ \\
        \hline
    \end{tabular}%
    }

\caption{
Comparison of trace estimators for a fixed nonzero symmetric positive
semidefinite matrix
$\mathbf A\in\mathbb R^{d^p\times d^p}$, where
$n=d^p$, $r=\operatorname{rank}(\mathbf A)$,
$\kappa_f(\mathbf A)
=\lambda_{\max}(\mathbf A)/\lambda_{\min}^{+}(\mathbf A)$, and
$r_D(\mathbf A)
=n\max_i A_{ii}/\operatorname{tr}(\mathbf A)$.
An $(\varepsilon,\delta)$-approximation $\widehat t$ satisfies
$\Pr[|\widehat t-\operatorname{tr}(\mathbf A)|
\le\varepsilon\operatorname{tr}(\mathbf A)]\ge1-\delta$.
The table highlights the trade-off among variance, randomness, and
the number of samples required for an
$(\varepsilon,\delta)$-approximation. Classical Hutchinson estimators
require $O(md^p)$ randomness, whereas Kronecker-Hutchinson estimators
reduce this requirement to $O(mpd)$ but incur an exponential
dependence on $p$ in both variance and sample complexity. In
contrast, the proposed Recursive TensorSketch estimators require
$O\!\bigl(p(d+m)\log m\bigr)$ randomness and have polynomial, rather
than exponential, dependence on $p$. 
}
    \label{tab:related_work_trace_estimators}
\end{table}

Trace estimation has a long history in randomized numerical linear algebra. The seminal algorithm by Girard and Hutchinson~\cite{Girard1987algorithme, article-hutchinson}, known as the Hutchinson trace estimator, provides an efficient method for approximating the trace. Given an implicit matrix $\mathbf{A} \in \mathbb{R}^{d^p \times d^p}$, the Hutchinson trace estimator is defined as $H(\mathbf{A}) = \mathbf{z}^\top \mathbf{A} \mathbf{z}$, where $\mathbf{z} \in \mathbb{R}^{d^p}$ is a random vector with i.i.d.\ entries, typically drawn from $\mathcal{N}(0,1)$ or a Rademacher distribution. The estimator satisfies $\mathbb{E}[H(\mathbf{A})] = \operatorname{tr}(\mathbf{A})$ and $\operatorname{Var}(H(\mathbf{A})) = 2\|\mathbf{A}\|_F^2$. To reduce the variance, the estimator is repeated independently $m$ times. Let $\mathbf{z}^{(1)}, \dots, \mathbf{z}^{(m)} \in \mathbb{R}^{d^p}$ be independent copies of $\mathbf{z}$, and define $H_m(\mathbf{A}) = \frac{1}{m} \sum_{i=1}^{m} (\mathbf{z}^{(i)})^\top \mathbf{A} \mathbf{z}^{(i)}$. Then, $\mathbb{E}[H_m(\mathbf{A})] = \operatorname{tr}(\mathbf{A})$ and $\operatorname{Var}(H_m(\mathbf{A})) = \frac{2}{m}\|\mathbf{A}\|_F^2$. Each query requires generating a random vector $\mathbf{z} \in \mathbb{R}^{d^p}$, which uses $O(d^p)$ random bits, leading to a total randomness of $O(md^p)$ for $m$ samples. 
 
In the classical setting, trace estimators are based on the form $X := \mathbf{z}^\top \mathbf{A} \mathbf{z}$, where $\mathbf{z}$ is a random query vector. Avron and Toledo~\cite{10.1145/1944345.1944349} study several such estimators that differ in the choice of the distribution of $\mathbf{z}$. In particular, they analyze the Hutchinson estimator under different choices of the query vector $\mathbf{z}$, including the case where its entries are \textit{i.i.d.} $\mathcal{N}(0,1)$, the variant with \textit{i.i.d.} Rademacher entries, and unit-vector-based estimators in which $\mathbf{z}$ is sampled uniformly from the standard basis.
They also study a mixed unit-vector estimator of the form 
$X_M := \mathbf{e}^\top \mathbf{F} \mathbf{A} \mathbf{F}^\top \mathbf{e}$, 
where each $\mathbf{e}$ is sampled uniformly from the standard basis vectors, and $\mathbf{F}$ is a fixed orthogonal transform (e.g. Hadamard matrix). 
Their work provides high-probability guarantees and highlights the trade-off between variance and the number of random bits used in these estimators.

% A recent work of Meyer et al.~\cite{meyer2025hutchinsonsestimatorbadkroneckertraceestimation} shows that the Kronecker Hutchinson estimator suffers from an exponential dependence on $p$ for variance. 
% In particular, when the estimator is averaged over $m$ independent samples, the variance of the real-valued version scales as $O\!\left(\frac{3^p}{m}\right)$, while the complex-valued version improves this to $O\!\left(\frac{2^p}{m}\right)$. 
% Although their approach reduces the randomness required to $O(mpd)$ by constructing query vectors as Kronecker products of $p$ independent vectors in $\mathbb{R}^d$, the exponential dependence on $p$ makes the estimator increasingly inaccurate as $p$ grows.

A recent work by Meyer et al.~\cite{meyer2025hutchinsonsestimatorbadkroneckertraceestimation} proposed a Kronecker-structured trace estimator to reduce the number of random bits required for trace estimation. Their query vector $\mathbf{x} \in \mathbb{R}^{d^p}$ is constructed as the Kronecker product of $p$ independent random vectors in $\mathbb{R}^d$, namely, $\mathbf{x} = \mathbf{x}_1 \otimes \cdots \otimes \mathbf{x}_p$, where $\mathbf{x}_i \in \mathbb{R}^d$ for each $i \in [p]$. Consequently, generating a single query vector $\mathbf{x}$ requires only $O(dp)$ random bits, and an estimator obtained by averaging $m$ independent samples requires $O(mdp)$ random bits. This is substantially smaller than the $O(md^p)$ random bits required by the classical Hutchinson trace estimator~\cite{article-hutchinson}. However, this reduction in randomness comes at the cost of increased variance. In particular, the variance of the estimator scales as $O\!\left(\frac{3^p}{m}\right)$, whereas a complex-valued variant improves this dependence to $O\!\left(\frac{2^p}{m}\right)$. Thus, although the Kronecker-structured approach significantly reduces the randomness requirement by constructing each query vector from $p$ independent vectors in $\mathbb{R}^d$, the exponential dependence of the variance on $p$ can make the estimator increasingly inaccurate as $p$ grows.

% Inspired by the work of Ahle et al.~\cite{doi:10.1137/1.9781611975994.9}, we propose a trace estimator based on \texttt{Recursive TensorSketch}. 
% For an implicit matrix $\mathbf{A} \in \mathbb{R}^{d^p \times d^p}$, our estimator $T(\mathbf{A}) := \operatorname{tr}\!\left(\Pi^{p} \mathbf{A} (\Pi^{p})^\top\right)$ is unbiased and admits a variance bound of $O\!\left(\frac{3}{m}\right)$, which is independent of $p$. 
% This improves exponentially the variance guarantees over Kronecker-Hutchinson estimators whose variance scales as $O\!\left(\frac{3^p}{m}\right)$ and $O\!\left(\frac{2^p}{m}\right)$ in the real and complex settings, respectively. 
% Moreover, the randomness complexity of our estimator is $O\!\big(p(d+m)\log m\big)$, which is asymptotically smaller than the $O(mpd)$ randomness required by Kronecker-Hutchinson estimator~\cite{meyer2025hutchinsonsestimatorbadkroneckertraceestimation} and exponentially smaller than the $O(md^p)$ randomness required by classical Hutchinson estimator~\cite{Girard1987algorithme, article-hutchinson}. 
% We further extend our framework to a complex-valued setting, obtaining improved variance bounds while retaining the same asymptotic randomness.

% In this work, we address this challenge by proposing an estimator that requires asymptotically fewer random bits than the estimator of~\cite{meyer2025hutchinsonsestimatorbadkroneckertraceestimation}, while maintaining a variance that grows only polynomially with $p$. 

In this work, we address this challenge by proposing an estimator that requires asymptotically fewer random bits while ensuring that its variance grows only polynomially with $p$.
Our work is inspired by the work of~\cite{doi:10.1137/1.9781611975994.9}, which proposed a recursive sketching algorithm \texttt{Recursive TensorSketch} for compressing polynomial kernels. We show that \texttt{Recursive TensorSketch} can also be leveraged to design a trace estimator that significantly reduces the number of random bits required while maintaining low variance.
% Our Inspired by the work of~\cite{doi:10.1137/1.9781611975994.9}, we propose a trace estimator based on \texttt{Recursive TensorSketch}. 
For an implicit matrix $\mathbf{A} \in \mathbb{R}^{d^p \times d^p}$, our estimator $T(\mathbf{A}) := \operatorname{tr}\!\left(\Pi^{p} \mathbf{A} (\Pi^{p})^\top\right)$ is unbiased and admits a variance bound of $O\!\left(\left(\frac{10p}{m} + \frac{100p^2}{m^2}\right)\bigl(\operatorname{tr}(A)\bigr)^2\right)$, which is polynomial of $p$. 
% Furthermore, when $\mathbf{A}$ is positive semi-definite, the relation $\|\mathbf{A}\|_F^2 \leq (\operatorname{tr}(\mathbf{A}))^2$ (see Remark~\ref{remark_on_trace_and_frobenius_relation}) implies that our bound is at most $O\!\left(\frac{3(\operatorname{tr}(\mathbf{A}))^2}{m}\right)$. 
Our estimator yields a exponential improvement over Kronecker-Hutchinson estimators~\cite{meyer2025hutchinsonsestimatorbadkroneckertraceestimation}, whose variance scales as $O\!\left(\frac{3^p(\operatorname{tr}(\mathbf{A}))^2}{m}\right)$ and $O\!\left(\frac{2^p(\operatorname{tr}(\mathbf{A}))^2}{m}\right)$ in the real and complex settings, respectively.
Moreover, the randomness complexity of our estimator is $O\!\big(p(d+m)\log m\big)$, which is asymptotically smaller than the $O(mpd)$ randomness required by Kronecker-Hutchinson estimator~\cite{meyer2025hutchinsonsestimatorbadkroneckertraceestimation} and exponentially smaller than the $O(md^p)$ randomness required by classical Hutchinson estimator~\cite{Girard1987algorithme, article-hutchinson}. 
We further extend our framework to a complex-valued setting, obtaining improved variance bounds with fewer random bits than the corresponding estimator of~\cite {meyer2025hutchinsonsestimatorbadkroneckertraceestimation}. 

A standard way to evaluate a randomized trace estimator is through an
$(\varepsilon,\delta)$-approximation guarantee
~\cite{10.1145/1944345.1944349}. For a fixed nonzero positive
semidefinite matrix $\mathbf{A}$, an estimator $\widehat{t}$ is called
an $(\varepsilon,\delta)$-approximation of
$\operatorname{tr}(\mathbf{A})$ if
\begin{align}
    \Pr\!\left[
        \left|
            \widehat{t}-\operatorname{tr}(\mathbf{A})
        \right|
        \leq
        \varepsilon\operatorname{tr}(\mathbf{A})
    \right]
    \geq
    1-\delta.
\end{align}
Here, $\varepsilon$ specifies the allowed relative error and $\delta$
specifies the failure probability. This guarantee is important because
it translates variance or concentration bounds into a required sample
or sketch size, thereby allowing different trace estimators to be
compared in terms of accuracy.
We summarize our comparison with the baseline methods in Table~\ref{tab:related_work_trace_estimators}, which highlights the trade-offs among variance, randomness, and the number of samples required to obtain an $(\varepsilon,\delta)$-approximation.

%% file: background.tex
\section{Background}
\textbf{Notation.}
We denote vectors by lowercase bold letters (e.g., $\mathbf{x}$) and matrices by uppercase bold letters (e.g., $\mathbf{M}$). For a matrix $\mathbf{M} \in \mathbb{R}^{n \times n}$, $\mathrm{tr}(\mathbf{M})$ denotes its trace and $\|\mathbf{M}\|_F$ its Frobenius norm. We write $\mathbf{M} \succeq 0$ to indicate that $\mathbf{M}$ is symmetric positive semi-definite. For a positive integer $d$, we denote $[d] := \{1,2,\dots,d\}$. Kronecker product is denoted by $\otimes$, and for vectors $\mathbf{x}_1, \dots, \mathbf{x}_p \in \mathbb{R}^d$, we write $\mathbf{x} = \mathbf{x}_1 \otimes \cdots \otimes \mathbf{x}_p \in \mathbb{R}^{d^p}$. We use $\mathbb{E}[\cdot]$ and $\mathrm{Var}(\cdot)$ to denote expectation and variance, respectively. Throughout the paper, $i \in [m]$ indexes sketch dimensions. For a complex vector or matrix in the field $ \mathbb{C}$ , we denote by $(\cdot)^*$ its conjugate transpose. Finally, $\mathds{1}[\cdot]$ denotes the indicator function.
% {\color{red} 152-157 can be removed if space crunch..}
%{\color{red} State both gaussian and rademacher estimator }
We first state the classical Hutchinson Trace Estimator and its concentration guarantee. 
\begin{theorem}[Hutchinson Trace Estimator~\cite{Girard1987algorithme, article-hutchinson}]
\label{lem:hutchinson}
Let $\mathbf{A} \in \mathbb{R}^{d^p \times d^p}$ be any implicit matrix. 
Then, the trace estimator is defined as
\(
H(\mathbf{A}) := \mathbf{z}^\top \mathbf{A} \mathbf{z}, \text{ where $\mathbf{z} \in \mathbb{R}^{d^p}$ such that  $z_i\sim\mathcal{N}(0,1)$.}
\)

% \begin{itemize}
%     \item If the entries of $\mathbf{y}$ are i.i.d.\ Rademacher random variables, i.e.,
%     \(
%     \mathbb{P}(y_i = +1) = \mathbb{P}(y_i = -1) = \tfrac{1}{2},
%     \)
%     then
%     \(
%     \mathbb{E}[H(\mathbf{M})] = \operatorname{tr}(\mathbf{M}), \quad 
%     \operatorname{Var}(H(\mathbf{M})) 
%     = 2\left( \|\mathbf{M}\|_F^2 - \sum_{i=1}^n M_{ii}^2 \right).
%     \)

%     \item If the entries of $\mathbf{y}$ are i.i.d.\ Gaussian random variables, i.e., $y_i \sim \mathcal{N}(0,1)$, then
%     \(
%     \mathbb{E}[H(\mathbf{M})] = \operatorname{tr}(\mathbf{M}), \quad 
%     \operatorname{Var}(H(\mathbf{M})) 
%     = 2\|\mathbf{M}\|_F^2.
%     \)
% \end{itemize}

% \medskip
% \noindent Now consider 

% an implicit matrix $\mathbf{A} \in \mathbb{R}^{d^p \times d^p}$. 

%To reduce variance, the estimator is repeated $m$ times independently. 
Let $\mathbf{z}^{(1)}, \dots, \mathbf{z}^{(m)} \in \mathbb{R}^{d^p}$ be \ {i.i.d.}  copies of $\mathbf{z}$, then the final estimator is defined as follows
\begin{align}\label{m_times_hut_estm}
    H_m(\mathbf{A}) &= \frac{1}{m} \sum_{i=1}^{m} \big(\mathbf{z}^{(i)}\big)^\top \mathbf{A} \mathbf{z}^{(i)}.\\
    \text{Then, }\quad \mathbb{E}[H_m(\mathbf{A})] = \operatorname{tr}(\mathbf{A}),& \quad 
\operatorname{Var}(H_m(\mathbf{A})) = \frac{2}{m}\|\mathbf{A}\|_F^2.
\end{align}
\end{theorem}

\begin{theorem}[High-Probability Error Bound~\cite{avron2010counting}]
\label{thm:avron_toledo_main}
Let $\mathbf{A} \succeq 0$ and let $H_{m}(\mathbf{A})$ be the estimator defined in Equation~\eqref{m_times_hut_estm} using 
Rademacher or Gaussian vectors. Then for any $\varepsilon, \delta \in (0,1)$,
it suffices to choose
\(
m = O\!\left( \frac{\log(1/\delta)}{\varepsilon^2} \right)
\)
samples to guarantee
\begin{align}
\mathrm{Pr}\!\left(
\left| H_{m}(\mathbf{A}) - \operatorname{tr}(\mathbf{A}) \right| 
\le \varepsilon \, \operatorname{tr}(\mathbf{A})
\right) \ge 1 - \delta.
\end{align}
\end{theorem}

\subsection{Trace Estimation via Kronecker-Matrix vector product}
\cite{meyer2025hutchinsonsestimatorbadkroneckertraceestimation} considered a variant of Hutchinson trace estimator where the problem is estimating the trace of an implicit  matrix $\mathbf{A} \in \R^{d^p\times d^p}$  that can only be accessed through Kronecker-matrix-vector products. That is, for any Kronecker-structured vector that is, $\mathbf{x} = \mathbf{x}_1 \otimes \cdots \otimes \mathbf{x}_p,$ where random vector  $\mathbf{x}_i \in \mathbb{R}^d$ for $i\in [p]$, Kronecker-matrix-vector product $\mathbf{Ax}$ can be computed. Their estimator is termed as \textit{Kronecker-Hutchinson estimator} and defined as follows:  
$
T := \mathbf{x}^\top \mathbf{A} \mathbf{x}.
$
They propose several estimators, each corresponding to different choices of distributions from which the random vectors 
$\mathbf{x}_i$ are sampled.

\begin{theorem}[Variance for real-valued Kronecker random vectors {\cite[Theorem 5.4]{meyer2025hutchinsonsestimatorbadkroneckertraceestimation}}]
\label{thm:kron_real}
Let $\mathbf{A} \in \mathbb{R}^{d^p \times d^p}$ be a PSD matrix. Let
$\mathbf{x} = \mathbf{x}_1 \otimes \cdots \otimes \mathbf{x}_p,$
where $\mathbf{x}_1, \dots, \mathbf{x}_p \in \mathbb{R}^d$ are independent and identically distributed random vectors. Then, all the following estimators are unbiased, and satisfy the following variance bounds 
\begin{align*}
\text{Gaussian:} \quad
&\text{if } \mathbf{x}_i \sim \mathcal{N}(\mathbf{0}, I_d), \\
&\operatorname{Var}[\mathbf{x}^\top \mathbf{A} \mathbf{x}]
\le 3^p \, (\operatorname{tr}(\mathbf{A}))^2, \\[6pt]
\text{Rademacher:} \quad
&\text{if entries of each } \mathbf{x}_i \text{ are i.i.d.\ in } \{-1,+1\}, \\
&\operatorname{Var}[\mathbf{x}^\top \mathbf{A} \mathbf{x}]
\le \left(3 - \frac{2}{d}\right)^p (\operatorname{tr}(\mathbf{A}))^2, \\[6pt]
\text{Uniform sphere:} \quad
&\text{if each } \mathbf{x}_i \text{ is drawn uniformly from } \mathbb{S}^{d-1}, \\
&\operatorname{Var}[\mathbf{x}^\top \mathbf{A} \mathbf{x}]
\le \left(3 - \frac{6}{d+2}\right)^p (\operatorname{tr}(\mathbf{A}))^2.
\end{align*}
\end{theorem}

\noindent The bounds in Theorem~\ref{thm:kron_real} exhibit an exponential dependence on the parameter $p$, rendering the Kronecker–Hutchinson estimator inefficient for large values of $p$. They further  demonstrate that using complex-valued random vectors leads to improved bounds with a smaller exponential factor. The following theorem summarizes those guarantees.

% \noindent The bounds in Theorem~\ref{thm:kron_real} exhibit an exponential dependence on the tensor {\color{red} there is no tensor order here..just state dependence on $p$...fix this here and every other place...}order $p$, which makes the Kronecker-Hutchinson estimator inefficient for large $p$ values. Meyer and Avron~\cite{meyer2025hutchinsonsestimatorbadkroneckertraceestimation} show that complex-valued vectors yield improved bounds with a smaller exponential factor. The following theorem summarizes these guarantees.

\begin{theorem}[Variance for Complex-Valued Structures (Theorem 6.2 and Lemma 6.3 of \cite{meyer2025hutchinsonsestimatorbadkroneckertraceestimation})]
\label{thm:kron_complex}
Let $\mathbf{A} \in \mathbb{R}^{d^p \times d^p}$ be a PSD matrix, and 
$\mathbf{x} = \mathbf{x}_1 \otimes \cdots \otimes \mathbf{x}_p$, where $\mathbf{x}_1, \dots, \mathbf{x}_p \in \mathbb{C}^d$ are i.i.d.\ random vectors.  Then, all the following estimators are unbiased, and satisfy the following variance bounds 
\begin{align*}
\text{Complex Gaussian:} \quad
&\text{if } \mathbf{x}_i = \tfrac{1}{\sqrt{2}}(\mathbf{r}_i + i \mathbf{m}_i), \text{ with } \mathbf{r}_i, \mathbf{m}_i \sim \mathcal{N}(\mathbf{0}, I_d),  \\
&\operatorname{Var}[\mathbf{x}^* \mathbf{A} \mathbf{x}]
\le 2^p \, (\operatorname{tr}(\mathbf{A}))^2, \\[6pt]
\text{Complex Rademacher:} \quad
&\text{if each entry of } \mathbf{x}_i \text{ is drawn i.i.d.\ from } \left\{\pm \tfrac{1}{\sqrt{2}}, \pm \tfrac{i}{\sqrt{2}}\right\}, \\
&\operatorname{Var}[\mathbf{x}^* \mathbf{A} \mathbf{x}]
\le \left(2 - \frac{1}{d}\right)^p (\operatorname{tr}(\mathbf{A}))^2, \\[6pt]
\text{Complex sphere:} \quad
&\text{if each } \mathbf{x}_i \text{ is uniformly distributed on the complex sphere }, \\
&\operatorname{Var}[\mathbf{x}^* \mathbf{A} \mathbf{x}]
\le \left(2 - \frac{2}{d+1}\right)^p (\operatorname{tr}(\mathbf{A}))^2.
\end{align*}
\end{theorem}

We address the limitations of the Kronecker–Hutchinson estimator by designing an alternative estimator that leverages \texttt{Recursive TensorSketch} proposed by~\cite{doi:10.1137/1.9781611975994.9}. In their work, the authors develop this technique in the context of sketching high-degree polynomial kernels, demonstrating that tensor product structures can be efficiently compressed via recursive linear-mappings while preserving inner-product similarity to a high degree of accuracy. We state their sketching algorithm in the following subsection.

\subsection{Introduction to \texttt{Recursive TensorSketch}}
\begin{figure}
    \centering
    \includegraphics[width=0.8\linewidth]{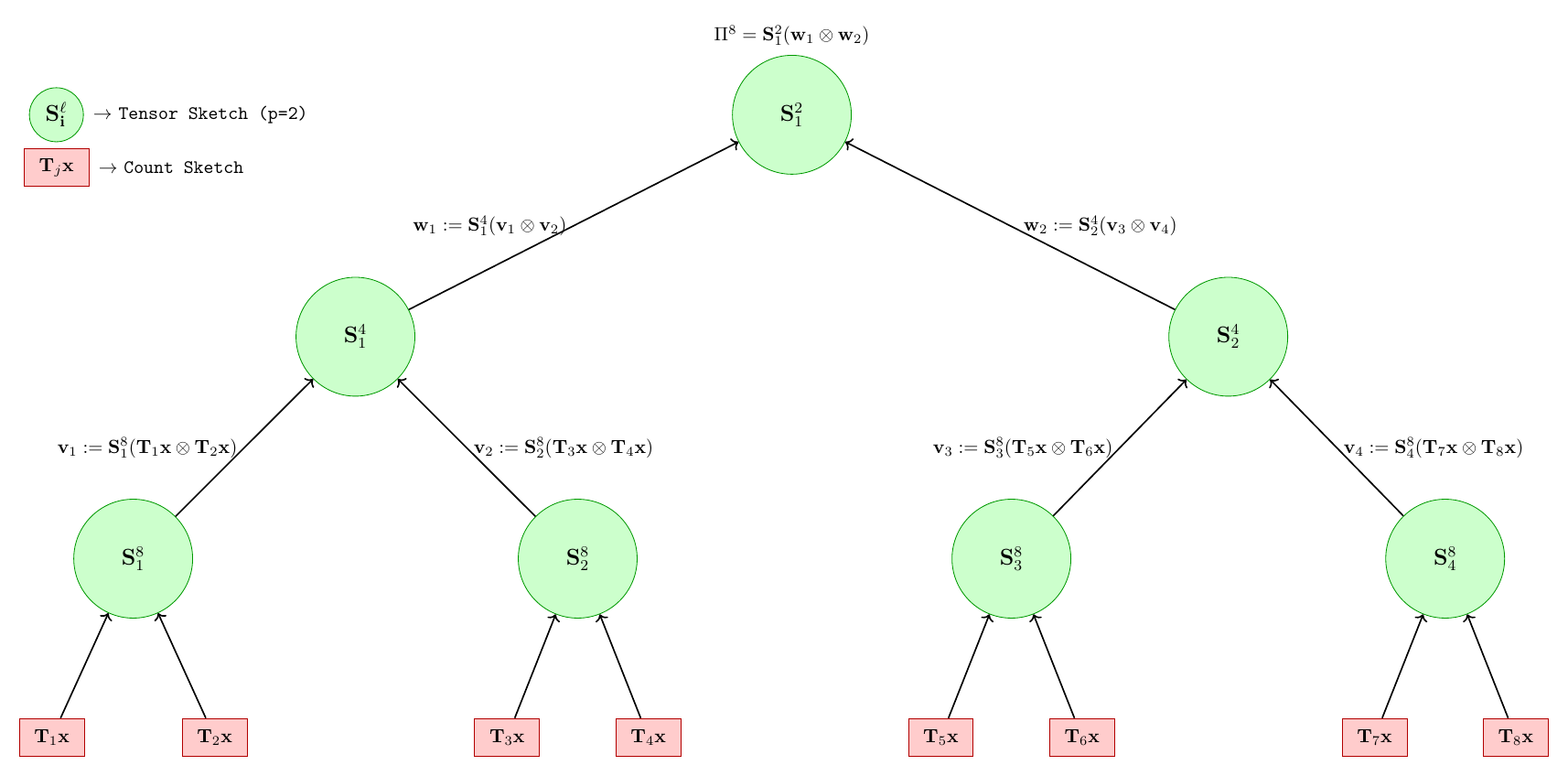}
    \caption{Recursive TensorSketch construction for $p=8$. 
    Each $T_j$ denotes a CountSketch, while $S_i^{\ell}$ denotes a TensorSketch 
    operator. Intermediate vectors are combined recursively.}
    \label{fig:recursive_tensorsketch}
\end{figure}
We begin by presenting the \texttt{CountSketch}~\cite{count_sketch} algorithm, which enables fast dimensionality reduction for high-dimensional vectors. We then describe \texttt{TensorSketch}~\cite{pham2013fast,DBLP:journals/corr/abs-2505-08146} of degree $2$, which extends this idea to efficiently compress vectors formed via the Kronecker product of two vectors.

\begin{definition}[\texttt{CountSketch}~\cite{count_sketch}] \label{count_sketch_def}
Given an input vector $\mathbf{y} \in \mathbb{R}^{d}$, the \texttt{CountSketch} is a randomized linear map $\mathbf{T} \in \mathbb{R}^{m \times d}$ that maps $\mathbf{y}$ to a lower-dimensional vector $\mathbf{z} = \mathbf{T}\,\mathbf{y} \in \mathbb{R}^m$. The \texttt{CountSketch} matrix $\mathbf{T}$ is constructed by two hash functions: (a) $h \colon [d] \to [m]$ a $3$-wise independent hash function, and (b) $s: [d] \to \{1,-1\}$ a $4$-wise independent random sign function. The $j^{th}$ entry of vector $\mathbf{z} \in \mathbb{R}^m$ is computed as,
$$z_j = \sum_{h(i) = j} s(i)\, y_i, \; \forall j \in \{1, \ldots, m\}.$$

\noindent The time complexity of computing the \texttt{CountSketch} is $O\!\left(\mathrm{nnz}(\mathbf{y})\right)$, which in the worst case can be $O(d)$.  Furthermore, \texttt{CountSketch} provides an unbiased estimator and variance of this estimator is
$\Var[\|\mathbf{Ty}\|_2^2] \leq \frac{2\|\mathbf{y}\|_2^4}{m}.$
\end{definition}
TensorSketch extends the idea of CountSketch to tensor products and allows them to be sketched efficiently.
\begin{definition}[\texttt{TensorSketch} of Degree Two~\cite{pham2013fast,DBLP:journals/corr/abs-2505-08146}] \label{def:TensorSketch of Degree Two}
Let $h_1, h_2 : [d] \to [m]$ be $3$-wise independent hash functions, and $\sigma_1, \sigma_2 : [d] \to \{-1, +1\}$ be $4$-wise independent random sign functions. Then the \texttt{TensorSketch} of degree two $\mathbf{S} \in \mathbb{R}^{m \times d^2}$ is defined $\forall \ r \in [m],~i_1, i_2 \in [d],$ as follows
\begin{equation}
S_{r,(i_1,i_2)} = \sigma_1(i_1)\cdot\sigma_2(i_2) \cdot  \mathds{1} \left[ h_1(i_1) + h_2(i_2) \equiv r \pmod m \right].
\end{equation}
\texttt{TensorSketch} provides an unbiased estimator of the squared $\ell_2$-norm, whose variance is bounded by $\Var\left[\|\mathbf{S}(\mathbf{x}\otimes\mathbf{x})\|_2^2\right]\le \frac{8\|\mathbf{x}\|_2^4}{m}$. Furthermore, for any $\mathbf{x}\in\mathbb{R}^d$, the sketch $\mathbf{S}(\mathbf{x}\otimes\mathbf{x})$ can be computed in $O\!\left(m\log m+\mathrm{nnz}(\mathbf{x})\right)$ time using the Fast Fourier Transform (FFT).

\end{definition}

Given a vector  $\mathbf{x} \in \R^{d^p}$ of the form $\mathbf{x}= \mathbf{x}_1 \otimes \cdots \otimes \mathbf{x}_p$, \texttt{Recursive TensorSketch} provides an efficient sketching procedure that avoids the explicit construction of $\mathbf{x}$. The method proceeds by first applying independent \texttt{CountSketch} transformations to each component vector $\mathbf{x}_i$,  for $i\in [p]$, and then recursively combining the resulting sketches using degree-two \texttt{TensorSketch} operations, producing a hierarchical tree-structured representation refer to Figure~\ref{fig:recursive_tensorsketch}.

\begin{definition}[\texttt{Recursive TensorSketch}~\cite{doi:10.1137/1.9781611975994.9}]
\label{def:Recursive_tensor_sketch_framework}

Given a vector $\mathbf{x} \in \mathbb{R}^{d^p}$, where $p$ is a power of two, the \texttt{Recursive TensorSketch} is a randomized linear map
\[
\Pi^p : \mathbb{R}^{d^p} \rightarrow \mathbb{R}^m, \quad \text{defined as} \quad \Pi^p := \mathbf{Q}^p \cdot \mathbf{T}^p, \text{~where}
\]

\begin{itemize}
    \item $\mathbf{T}^p = \mathbf{T}_{1} \otimes \mathbf{T}_{2} \otimes \cdots \otimes \mathbf{T}_p$, with each $T_{i} \in \mathbb{R}^{m \times d}$ for $i \in [p]$ a \texttt{CountSketch} matrix (Definition~\ref{count_sketch_def}),
    
    \item $\mathbf{Q}^p = \mathbf{S}^2 \cdot \mathbf{S}^4 \cdots \mathbf{S}^{p/2} \cdot \mathbf{S}^p$, with each $\mathbf{S}^\ell \in \mathbb{R}^{m^{\ell/2} \times m^\ell}$ a Kronecker product of matrices $S_j^\ell \in \mathbb{R}^{m \times m^2}$,
    
    \item each $\mathbf{S}_j^\ell$ is a \texttt{TensorSketch} matrix of degree $2$ (Definition~\ref{def:TensorSketch of Degree Two}), and $\mathbf{S}^\ell = \mathbf{S}_1^\ell \otimes \mathbf{S}_2^\ell \otimes \cdots \otimes \mathbf{S}_{\ell/2}^\ell$.
\end{itemize}

\noindent When $\mathbf{x}$ is given in the form of Kronecker product of $p$ vectors, i.e., $\mathbf{x} = \mathbf{x}_1 \otimes \cdots \otimes \mathbf{x}_p$ with $\mathbf{x}_i \in \mathbb{R}^d$ for all $i \in [p]$, the \texttt{Recursive TensorSketch} can be computed efficiently in time $O(p\,m \log m + p d)$. In contrast, when $\mathbf{x}$ is an arbitrary vector in $\mathbb{R}^{d^p}$ without explicit Kronecker structure, computing $\Pi^p \mathbf{x}$ requires $O(m d^p)$ time.
\end{definition}

\begin{definition}[Definition 18 of~\cite{doi:10.1137/1.9781611975994.9}:
                   JL Moment Property]
\label{def:jl_moment_property}
For every positive integer $t$ and every $\delta,\varepsilon\ge 0$, a
distribution over random matrices $\mathbf{M}\in\R^{m\times d}$ has the
\emph{$(\varepsilon,\delta,t)$-JL Moment Property} if
\begin{align*}
   \bigl\|\,\|\mathbf{M}\mathbf{x}\|_{2}^{2}-1\,\bigr\|_{L^{t}}\;\le\;\varepsilon\,\delta^{1/t}
   \qquad\text{and}\qquad
   \mathbb{E}\!\left[\|\mathbf{M}\mathbf{x}\|_{2}^{2}\right]\;=\;1
\end{align*}
for every unit vector $\mathbf{x}\in\R^{d}$.
\end{definition}

\noindent Now, we state some useful results from~\cite{doi:10.1137/1.9781611975994.9} that will be used in our proofs.

\begin{lemma}[Lemma 9 of~\cite{doi:10.1137/1.9781611975994.9}:
              Two-vector JL Moment Property]
\label{lem:two_vector_jl}
For any $\mathbf{x},\mathbf{y}\in\R^{d}$, if $\mathbf{M}$ has the
$(\varepsilon,\delta,t)$-JL Moment Property, then
\[
   \bigl\|\langle \mathbf{M}\mathbf{x},\,\mathbf{M}\mathbf{y}\rangle 
        -\langle\mathbf{x},\,\mathbf{y}\rangle\bigr\|_{L^{t}}
   \;\le\;\varepsilon\,\delta^{1/t}\,\|\mathbf{x}\|_{2}\|\mathbf{y}\|_{2}.
\]
\end{lemma}

\begin{lemma}[Lemma 12 of~\cite{doi:10.1137/1.9781611975994.9}:
              Factorisation of $\Pi^{p}$]
\label{lem:factorisation}
For any integer $p$ which is a power of two, let
$\Pi^{p}:\R^{d^{p}}\to\R^{m}$ be \texttt{Recursive TensorSketch} defined in Definition~\ref{def:Recursive_tensor_sketch_framework}, for sketches
$\mathbf{S}_i^\ell:\R^{m^{2}}\to\R^{m}$ and 
$\mathbf{T}_j:\R^{d}\to\R^{m}$.
Then there exist matrices $\bigl(\mathbf{M}^{(i)}\bigr)_{i\in[p-1]}$,
$\bigl(\mathbf{M'}^{(j)}\bigr)_{j\in[p]}$ and integers $(k_i)_{i \in [p-1]}$, $(k'_i)_{i \in [p-1]}$, $(l_j)_{j \in [p]}$, $(l'_j)_{j \in [p]}$, such that
\[
   \Pi^{p}\;=\;\mathbf{M}^{(p-1)}\cdots \mathbf{M}^{(1)}\cdot \mathbf{M}'^{(p)}\cdots \mathbf{M}'^{(1)},
\]
where $\mathbf{M}^{(i)} = I_{k_{i}}\otimes \mathbf{S}_i^\ell\otimes I_{k'_{i}}$
and $\mathbf{M'}^{(j)} = I_{\ell_{j}}\otimes \mathbf{T}_j\otimes I_{\ell'_{j}}$,
with $\mathbf{S}_i^\ell$ and $\mathbf{T}_j$ independent
instances of \texttt{TensorSketch} of Degree-$2$ and \texttt{CountSketch}, respectively,
for every $i\in[p-1]$ and $j\in[p]$.
\end{lemma}

\begin{lemma}[Lemma 14 of~\cite{doi:10.1137/1.9781611975994.9}:
              JL Moment Property under tensor wraps]
\label{lem:tensor_wrap_preserves_jl}
If the matrix $\mathbf{S}$ has the $(\varepsilon,\delta,t)$-JL Moment Property,
then for any positive integers $k,k'$, the matrix
$\mathbf{M} = I_{k}\otimes \mathbf{S}\otimes I_{k'}$ has the
$(\varepsilon,\delta,t)$-JL Moment Property.
\end{lemma}

\begin{lemma}[Lemma 15 of~\cite{doi:10.1137/1.9781611975994.9}:
              Composition lemma for the second moment]
\label{lem:composition_second_moment}
For any $\varepsilon,\delta\ge 0$ and any integer $k$, if 
$\mathbf{M}^{(1)}\in\R^{d_{2}\times d_{1}},\ldots,\mathbf{M}^{(k)}\in\R^{d_{k+1}\times d_{k}}$
are independent random matrices each with the
$\bigl(\tfrac{\varepsilon}{\sqrt{2k}},\,\delta,\,2\bigr)$-JL Moment
Property, then the product matrix
$\mathbf{M} = \mathbf{M}^{(k)}\cdots \mathbf{M}^{(1)}$ satisfies the $(\varepsilon,\delta,2)$-JL
Moment Property.
\end{lemma}

\begin{corollary}[Corollary 16 of~\cite{doi:10.1137/1.9781611975994.9}:
                  Second moment property for $\Pi^{p}$]
\label{cor:jl_for_recursive_sketch}
For any power-of-two integer $p$, let $\Pi^{p}:\R^{d^{p}}\to\R^{m}$ be
defined in Definition~\ref{def:Recursive_tensor_sketch_framework},
where both base distributions $S_i^\ell:\R^{m^{2}}\to\R^{m}$
and $T_j:\R^{d}\to\R^{m}$ satisfy the
$\bigl(\tfrac{\varepsilon}{\sqrt{4p+2}},\,\delta,\,2\bigr)$-JL Moment
Property. Then $\Pi^{p}$ satisfies the
$\left(\varepsilon,\delta,2\right)$-JL Moment Property.
\end{corollary}

%% file: real_rts_analysis.tex
The exponential variance growth of the Kronecker-Hutchinson estimator motivates alternative approaches for tensor-structured trace estimation. Although complex-valued structures offer partial improvement, they do not remove this dependence. The \texttt{Recursive TensorSketch} provides a structured way to compress tensor products, suggesting a more efficient estimator. We introduced trace estimator using \texttt{Recursive TensorSketch} and analyse its variance in the following section.

\section{Trace Estimator using \texttt{Recursive TensorSketch}}

In this section, Definition~\ref{def:rts_trace_estimator} introduces a
\texttt{Recursive TensorSketch}-based trace estimator for positive
semidefinite matrices given in implicit form. Our analysis uses the
independent-layer factorization established in
Lemma~\ref{lem:factorisation}. Lemma~\ref{lem:single_layer_base_trace}
first establishes expectation and variance bounds for a single sketch
satisfying the second-moment JL property, and
Lemma~\ref{lem:wrapped_layer_trace} extends these bounds to the
Kronecker-wrapped sketching operators arising in the recursive
construction. Theorem~\ref{thm:rts_trace_estimator} then applies these
bounds conditionally across the $2p-1$ independent layers to prove
unbiasedness and derive the variance bound. Finally,
Lemma~\ref{lem:rts_randomness} analyzes the number of random bits
required to construct the estimator, and
Theorem~\ref{concentration_analysis_RTS} presents the corresponding
concentration guarantee.

\begin{definition}[Recursive TensorSketch (RTS) Trace Estimator]
\label{def:rts_trace_estimator}
Let $\mathbf{\Pi}^{p} \in \mathbb{R}^{m \times d^p}$ denote the Recursive TensorSketch matrix stated in Definition~\ref{def:Recursive_tensor_sketch_framework}. For an implicit PSD matrix
$\mathbf{A}\in\mathbb{R}^{d^p\times d^p}$, its trace estimator is defined as follows:
\begin{align}
T(\mathbf{A}) := \operatorname{tr}\!\big(\mathbf{\Pi}^{p} \mathbf{A} (\mathbf{\Pi}^{p})^\top\big).
\end{align}
\end{definition}

\begin{lemma}[Expectation and Variance bound for a single-layer of Sketch]
\label{lem:single_layer_base_trace}
Let $\mathbf{S}\in\mathbb{R}^{m\times d_i}$ be matrix satisfying the
$(\varepsilon_0^{(i)},\delta_0^{(i)},2)$-JL Moment Property (Definition~\ref{def:jl_moment_property}), with
\(
    \bigl(\varepsilon_0^{(i)}\bigr)^2\delta_0^{(i)}
    \le \frac{c_i}{m}.
\)
Then, for every positive semidefinite matrix
$\widetilde{\mathbf{B}}\in\mathbb{R}^{d_i\times d_i}$,
\begin{equation}
    \mathbb{E}\!\left[
        \operatorname{tr}(\mathbf{S}\widetilde{\mathbf{B}} \mathbf{S}^\top)
    \right]
    =
    \operatorname{tr}(\widetilde{\mathbf{B}}), \text{ and } \operatorname{Var}\!\left(
        \operatorname{tr}(\mathbf{S}\widetilde{\mathbf{B}} \mathbf{S}^\top)
    \right)
    \le
    \frac{c_i}{m}
    \bigl(\operatorname{tr}(\widetilde{\mathbf{B}})\bigr)^2.
\end{equation}
\end{lemma}

\begin{proof}
The proof uses the eigendecomposition of
$\widetilde{\mathbf{B}}$. The unbiasedness follows from linearity of
the trace, while the variance bound follows by applying the
second-moment JL property to each eigenvector and then using
Minkowski's inequality.

Let
\begin{align*}
    \widetilde{\mathbf{B}}
    =
    \sum_{r=1}^{R}\lambda_r\mathbf{u}_r\mathbf{u}_r^\top
\end{align*}
be an eigendecomposition of $\widetilde{\mathbf{B}}$, where
$\lambda_r\geq0$ and
$\{\mathbf{u}_r\}_{r=1}^{R}$ is an orthonormal set. By linearity of
the trace,
\begin{align*}
    \operatorname{tr}
    \bigl(\mathbf{S}\widetilde{\mathbf{B}}\mathbf{S}^\top\bigr)
    =
    \sum_{r=1}^{R}
    \lambda_r\|\mathbf{S}\mathbf{u}_r\|_2^2.
\end{align*}
The unbiasedness property of $\mathbf{S}$ gives
\(
    \mathbb{E}\!\left[
        \|\mathbf{S}\mathbf{u}_r\|_2^2
    \right]
    =1.
    \label{eq:unbiased_S}
\)
Therefore,
\begin{align*}
    \mathbb{E}\!\left[
        \operatorname{tr}
        \bigl(\mathbf{S}\widetilde{\mathbf{B}}\mathbf{S}^\top\bigr)
    \right]
    &=
    \sum_{r=1}^{R}\lambda_r
    \mathbb{E}\!\left[
        \|\mathbf{S}\mathbf{u}_r\|_2^2
    \right]\\
    &=
    \sum_{r=1}^{R}\lambda_r
    =
    \operatorname{tr}(\widetilde{\mathbf{B}}).
\end{align*}

Define
\begin{align*}
    Z_r
    :=
    \|\mathbf{S}\mathbf{u}_r\|_2^2-1.
\end{align*}
By  unbiasedness property,  $\mathbb{E}[Z_r]=0$, and the
second-moment JL property gives
\begin{align*}
    \|Z_r\|_{L_2}
    \leq
    \varepsilon_0^{(i)}
    \bigl(\delta_0^{(i)}\bigr)^{1/2}.
\end{align*}
    Here, for any random variable $Y$,
\(
    \|Y\|_{L_2}
    :=
    \left(\mathbb{E}[|Y|^2]\right)^{1/2}
\)
denotes its $L_2$ norm. Since
\begin{align*}
    \operatorname{tr}
    \bigl(\mathbf{S}\widetilde{\mathbf{B}}\mathbf{S}^\top\bigr)
    -
    \operatorname{tr}(\widetilde{\mathbf{B}})
    =
    \sum_{r=1}^{R}\lambda_r Z_r,
\end{align*}
Minkowski's inequality, which is the triangle inequality for the
$L_2$ norm, gives
\begin{align*}
    \left\|
        \sum_{r=1}^{R}\lambda_rZ_r
    \right\|_{L_2}
    &\leq
    \sum_{r=1}^{R}\lambda_r\|Z_r\|_{L_2}\\
    &\leq
    \varepsilon_0^{(i)}
    \bigl(\delta_0^{(i)}\bigr)^{1/2}
    \sum_{r=1}^{R}\lambda_r\\
    &=
    \varepsilon_0^{(i)}
    \bigl(\delta_0^{(i)}\bigr)^{1/2}
    \operatorname{tr}(\widetilde{\mathbf{B}}).
\end{align*}
This argument does not require the random variables $Z_r$ to be
independent. Since the sum $\sum_{r=1}^{R}\lambda_rZ_r$ has mean zero,
its squared $L_2$ norm equals its variance. Squaring the preceding
inequality therefore yields
\begin{align*}
    \operatorname{Var}\!\left(
        \operatorname{tr}
        \bigl(\mathbf{S}\widetilde{\mathbf{B}}\mathbf{S}^\top\bigr)
    \right)
    &\leq
    \bigl(\varepsilon_0^{(i)}\bigr)^2
    \delta_0^{(i)}
    \bigl(\operatorname{tr}(\widetilde{\mathbf{B}})\bigr)^2\\
    &\leq
    \frac{c_i}{m}
    \bigl(\operatorname{tr}(\widetilde{\mathbf{B}})\bigr)^2.
\end{align*}
\end{proof}

Having established the expectation and variance bounds for a single
sketching layer in Lemma~\ref{lem:single_layer_base_trace}, we now state
the main guarantee and subsequently prove it for the complete \texttt{Recursive TensorSketch}
trace estimator.
\begin{theorem}[Unbiasedness and Variance of RTS Trace Estimator]
\label{thm:rts_trace_estimator}
Let $T(\mathbf{A})$ be the estimator defined in 
Definition~\ref{def:rts_trace_estimator}. Then,
\begin{align}
\mathbb{E}[T(\mathbf{A})] &= \operatorname{tr}(\mathbf{A}), \quad \text{and} \quad
\operatorname{Var}(T(\mathbf{A})) \le\left(\frac{10p}{m} + \frac{100p^{2}}{m^{2}}\right) \,\bigl(\operatorname{tr}(\mathbf{A})\bigr)^2.
\end{align}
\end{theorem}

% {\color{red} Add lines ...something like the following....recall that recursive tensor sketch matrix $\mathbf{\Pi}$ can be written as product of $(2p-1)$ matrices from Lemma 6. We first compute the variance of one sketch and subsequently upper bound the variance due to $\mathbf{\Pi}$ via a composition argument. Pls explain what is $\mathbf{I}_{d_a}$ in the statement below. Its notation is not consistent. Some places it is normal font.}

\noindent To prove Theorem~\ref{thm:rts_trace_estimator}, recall from
Lemma~\ref{lem:factorisation} that the \texttt{Recursive TensorSketch}
matrix $\mathbf{\Pi}^{p}$ can be expressed as a product of $2p-1$
independent sketching matrices. We first establish expectation and
variance bounds for one Kronecker-wrapped sketching layer in Lemma~\ref{lem:single_layer_base_trace}. We then
apply these bounds successively to all $2p-1$ layers through a
composition argument to obtain the variance bound for
$\mathbf{\Pi}^{p}$.

\begin{lemma}[Unbiasedness and Variance Guarantees for a Layer]
\label{lem:wrapped_layer_trace}
Let $\mathbf{S}\in\mathbb{R}^{m\times d_i}$ satisfy the assumptions of
Lemma~\ref{lem:single_layer_base_trace}. Let $d_a$ and $d_b$ be positive
integers, and let
$\mathbf{I}_{d_a}\in\mathbb{R}^{d_a\times d_a}$ and
$\mathbf{I}_{d_b}\in\mathbb{R}^{d_b\times d_b}$ denote the identity
matrices acting on the tensor components before and after the component
sketched by $\mathbf{S}$, respectively. Define
\[
    \mathbf{M}^{(i)}
    =
    \mathbf{I}_{d_a}\otimes\mathbf{S}\otimes\mathbf{I}_{d_b}.
\]
Then, for every
positive semidefinite matrix
\(
    \mathbf{B}\in
    \mathbb{R}^{d_a d_i d_b\times d_a d_i d_b},
\)
we have
\begin{align}
    \mathbb{E}\!\left[
        \operatorname{tr}\!\left(
            \mathbf{M}^{(i)}
            \mathbf{B}
            (\mathbf{M}^{(i)})^\top
        \right)
    \right]
    &=
    \operatorname{tr}(\mathbf{B}),\\
    \operatorname{Var}\!\left(
        \operatorname{tr}\!\left(
            \mathbf{M}^{(i)}
            \mathbf{B}
            (\mathbf{M}^{(i)})^\top
        \right)
    \right)
    &\leq
    \frac{c_i}{m}
    \bigl(\operatorname{tr}(\mathbf{B})\bigr)^2.
\end{align}
\end{lemma}

\begin{proof}
Let matrix of suitable dimension be $\mathbf{B} \in \mathbb{R}^{d_ad_id_b \times d_ad_id_b}$ where  $d_a$, $d_i$, and $d_b$ are suitable dimensions.

Let the full space be indexed by the tuple $(a, j, b)$ corresponding to the dimensions $d_a$, $d_i$, and $d_b$ respectively. We partition the matrix $\mathbf{B}$ into blocks $\mathbf{B}^{(a, a', b, b')} \in \mathbb{R}^{d_i \times d_i}$ by fixing the outer dimensions at indices $(a, a') \in [d_a]^2$ and $(b, b') \in [d_b]^2$. 

The sketching operator at layer $i$ is defined as the Kronecker product:
\begin{align*}
    \mathbf{M}^{(i)} = \mathbf{I}_{d_a} \otimes \mathbf{S} \otimes \mathbf{I}_{d_b}.
\end{align*}
Given that the base sketch $\mathbf{S} \in \mathbb{R}^{m \times d_i}$ and the identity matrices are $\mathbf{I}_{d_a} \in \mathbb{R}^{d_a \times d_a}$ and $\mathbf{I}_{d_b} \in \mathbb{R}^{d_b \times d_b}$, the dimensions of $\mathbf{M}^{(i)}$ multiply across the tensor product. Thus, $\mathbf{M}^{(i)}$ has dimensions:    $\mathbf{M}^{(i)} \in \mathbb{R}^{(d_a m d_b) \times (d_a d_i d_b)}$.
\noindent When we sketch $\mathbf{B}$ using $\mathbf{M}^{(i)}$, the matrix multiplication aligns as follows:
\begin{itemize}
    \item $\mathbf{M}^{(i)}$ is of size $(d_a m d_b) \times (d_a d_i d_b)$
    \item $\mathbf{B}$ is of size $(d_a d_i d_b) \times (d_a d_i d_b)$
    \item $(\mathbf{M}^{(i)})^\top$ is of size $(d_a d_i d_b) \times (d_a m d_b)$
\end{itemize}

\noindent We now partition $\mathbf{B}$ into $d_i \times d_i$ blocks denoted by $\mathbf{B}^{(a, a', b, b')}$, such that:
\begin{align*}
    \mathbf{B} = \sum_{a, a'=1}^{d_a} \sum_{b, b'=1}^{d_b} (\mathbf{e}_a \mathbf{e}_{a'}^\top) \otimes \mathbf{B}^{(a, a', b, b')} \otimes (\mathbf{e}_b \mathbf{e}_{b'}^\top).
\end{align*}

\noindent We now apply the sketching operator $\mathbf{M}^{(i)}$ to $\mathbf{B}$. Using the mixed-product property of Kronecker products, $(\mathbf{X} \otimes \mathbf{Y})(\mathbf{U} \otimes \mathbf{V}) = (\mathbf{X}\mathbf{U} \otimes \mathbf{Y}\mathbf{V})$, we obtain:
\begin{align*}
    \mathbf{M}^{(i)} \mathbf{B} (\mathbf{M}^{(i)})^\top 
    &= \left( \mathbf{I}_{d_a} \otimes \mathbf{S} \otimes \mathbf{I}_{d_b} \right) \mathbf{B} \left( \mathbf{I}_{d_a} \otimes \mathbf{S}^\top \otimes \mathbf{I}_{d_b} \right) \\
    &= \sum_{a, a', b, b'} \left(\mathbf{I}_{d_a} \mathbf{e}_a \mathbf{e}_{a'}^\top \mathbf{I}_{d_a}\right) \otimes \left(\mathbf{S} \mathbf{B}^{(a, a', b, b')} \mathbf{S}^\top\right) \otimes \left(\mathbf{I}_{d_b} \mathbf{e}_b \mathbf{e}_{b'}^\top \mathbf{I}_{d_b}\right) \\
    &= \sum_{a, a', b, b'} (\mathbf{e}_a \mathbf{e}_{a'}^\top) \otimes \left(\mathbf{S} \mathbf{B}^{(a, a', b, b')} \mathbf{S}^\top\right) \otimes (\mathbf{e}_b \mathbf{e}_{b'}^\top).
\end{align*}

\noindent Finally, we apply the trace operator. The trace of a Kronecker product is the product of the traces, i.e., $\operatorname{tr}(\mathbf{X} \otimes \mathbf{Y}) = \operatorname{tr}(\mathbf{X})\operatorname{tr}(\mathbf{Y})$. Applying this to our summation gives:
\begin{align}
    \operatorname{tr}\big(\mathbf{M}^{(i)} \mathbf{B} (\mathbf{M}^{(i)})^\top\big) 
    &= \sum_{a, a'=1}^{d_a} \sum_{b, b'=1}^{d_b} \operatorname{tr}(\mathbf{e}_a \mathbf{e}_{a'}^\top) \cdot \operatorname{tr}\big(\mathbf{S} \mathbf{B}^{(a, a', b, b')} \mathbf{S}^\top\big) \cdot \operatorname{tr}(\mathbf{e}_b \mathbf{e}_{b'}^\top).
\end{align}

\noindent Recall that the trace of an outer product of basis vectors is the inner product of the vectors: $\operatorname{tr}(\mathbf{e}_a \mathbf{e}_{a'}^\top) = \langle \mathbf{e}_{a'}, \mathbf{e}_a \rangle $. Thus, $\operatorname{tr}(\mathbf{e}_a \mathbf{e}_{a'}^\top)$ is $1$ if $a = a'$ and $0$ otherwise.

Therefore we can write:
\begin{align}
    \operatorname{tr}\big(\mathbf{M}^{(i)} \mathbf{B} (\mathbf{M}^{(i)})^\top\big) 
    &= \sum_{a=1}^{d_a} \sum_{b=1}^{d_b} \operatorname{tr}\big(\mathbf{S} \mathbf{B}^{(a, a, b, b)} \mathbf{S}^\top\big).\label{eq:block_trace_sum}
\end{align}

Using the linearity of matrix addition, we define a matrix $\tilde{\mathbf{B}} \in \mathbb{R}^{d_i \times d_i}$ on the single subspace where the random sketch $\mathbf{S}$ operates:
\begin{align}
    \tilde{\mathbf{B}} := \sum_{a=1}^{d_a} \sum_{b=1}^{d_b} \mathbf{B}^{(a, a, b, b)}.
\end{align}
Substituting $\tilde{\mathbf{B}}$ back into Equation~\eqref{eq:block_trace_sum}, the trace of the entire high-dimensional Kronecker layer simplifies to a standard matrix sketch trace on the lower-dimensional space:
\begin{align}
    \operatorname{tr}(\mathbf{M}^{(i)} \mathbf{B} (\mathbf{M}^{(i)})^\top) = \operatorname{tr}(\mathbf{S} \tilde{\mathbf{B}} \mathbf{S}^\top). \label{eq:simplified_trace_s}
\end{align}

\noindent Applying Lemma~\ref{lem:single_layer_base_trace} and evaluating it further gives

\begin{align}
    \mathbb{E}\!\left[
        \operatorname{tr}
        \bigl(\mathbf{M}^{(i)}\mathbf{B}(\mathbf{M}^{(i)})^\top\bigr)
    \right]
    &= \mathbb{E}\!\left[
        \operatorname{tr}(\mathbf{S} \tilde{\mathbf{B}} \mathbf{S}^\top)
    \right] = 
    \operatorname{tr}(\widetilde{\mathbf{B}})
    \\
    &=
    \operatorname{tr}\left(
        \sum_{a=1}^{d_a}\sum_{b=1}^{d_b}
        \mathbf{B}^{(a,a,b,b)}
    \right) =
    \sum_{a=1}^{d_a}\sum_{b=1}^{d_b}
    \operatorname{tr}\left(\mathbf{B}^{(a,a,b,b)}\right)
    \\
    &=
    \sum_{a=1}^{d_a}\sum_{b=1}^{d_b}
    \sum_{j=1}^{d_i}
    \mathbf{B}_{(a,j,b),(a,j,b)}=
    \operatorname{tr}(\mathbf{B}),
\end{align}
and similarly,
\begin{align}
    \operatorname{Var}\!\left(
        \operatorname{tr}
        \bigl(\mathbf{M}^{(i)}\mathbf{B}(\mathbf{M}^{(i)})^\top\bigr)
    \right)
    &=
    \operatorname{Var}\!\left(
        \operatorname{tr}(\mathbf{S}\widetilde{\mathbf{B}} \mathbf{S}^\top)
    \right) \\
    &\le
    \frac{c_i}{m}
    \bigl(\operatorname{tr}(\widetilde{\mathbf{B}})\bigr)^2 \\
    &=
    \frac{c_i}{m}
    \bigl(\operatorname{tr}(\mathbf{B})\bigr)^2.
\end{align}
\end{proof}

\noindent We now
apply bounds of single layer established in Lemma~\ref{lem:wrapped_layer_trace} successively to all $2p-1$ layers through a
composition argument to obtain the variance bound for
$\mathbf{\Pi}^{p}$.

\subsection{Proof of Theorem~\ref{thm:rts_trace_estimator} via Composition}
\label{proof_rts_variance_bounds}

The proof of Theorem~\ref{thm:rts_trace_estimator} relies on the factorization of the
\texttt{Recursive TensorSketch} matrix $\mathbf{\Pi}^p$ into a sequence of
mutually independent random sketching matrices. The argument applies
the single-layer trace moment bounds conditionally at each layer and
uses the PSD structure of the input matrix.

\begin{proof}
Let $\mathbf{A}\succeq 0$ and let $k=2p-1$. By
Lemma~\ref{lem:factorisation}, the \texttt{Recursive TensorSketch}
matrix admits the independent-layer factorization
\begin{align*}
    \mathbf{\Pi}^p
    =
    \mathbf{M}^{(k)}\mathbf{M}^{(k-1)}\cdots \mathbf{M}^{(1)}.
\end{align*}
Here, the first $p$ factors correspond to the \texttt{CountSketch}
maps at the leaf level, while the remaining $p-1$ factors correspond
to the degree-$2$ \texttt{TensorSketch} maps at the internal nodes of
the recursive tree.

\noindent Define
\begin{align*}
    \mathbf{A}_0 &:= \mathbf{A},\\
     \mathbf{A}_i &:=  \mathbf{M}^{(i)} \mathbf{A}_{i-1}( \mathbf{M}^{(i)})^\top,
    \qquad i=1,\ldots,k,
\end{align*}
and let
\begin{align*}
    X_i:=\operatorname{tr}(\mathbf{A}_i).
\end{align*}
Since $\mathbf{A}_0\succeq0$ and each matrix $\mathbf{A}_i$ is obtained
from $\mathbf{A}_{i-1}$ by multiplication with $\mathbf{M}^{(i)}$ and
its transpose, positive semidefiniteness is preserved at every step.
Therefore,
\begin{align}
    \mathbf{A}_i\succeq0
    \qquad\text{for every }i=0,\ldots,k.
    \label{eq:Ai_psd}
\end{align}

Let $\mathcal{F}_{i-1}$ represent all the random choices made in the
first $i-1$ layers. Putting Condition on this, the matrix
$\mathbf{A}_{i-1}$ is fixed and positive semidefinite, while
$\mathbf{M}^{(i)}$ remains independent and random. Therefore,
Lemma~\ref{lem:wrapped_layer_trace} gives
\begin{align}
    \mathbb{E}[X_i\mid\mathcal{F}_{i-1}]
    &=
    X_{i-1},
    \label{eq:conditional_trace_expectation}\\
    \operatorname{Var}(X_i\mid\mathcal{F}_{i-1})
    &\leq
    \frac{c_i}{m}X_{i-1}^2.
    \label{eq:conditional_trace_variance}
\end{align}
where, by the variance of \texttt{CountSketch} and 
\texttt{TensorSketch} of degree-$2$ given in
Definitions~\ref{count_sketch_def} and~\ref{def:TensorSketch of Degree Two},
respectively,
\begin{align*}
    c_i=
    \begin{cases}
        2, & i=1,\ldots,p,\\
        8=3^2-1, & i=p+1,\ldots,2p-1.
    \end{cases}
\end{align*}
Here, $2$ and $8$ are the corresponding second-moment JL constants.
    
\noindent \textbf{Expectation.}
Taking expectations in
Equation~\eqref{eq:conditional_trace_expectation} and applying the tower
property yields
\begin{align*}
    \mathbb{E}[X_i]
    &=
    \mathbb{E}[X_{i-1}].
\end{align*}
Iterating over all $k$ layers gives
\begin{align*}
    \mathbb{E}[X_k]
    =
    \mathbb{E}[X_0]
    =
    \operatorname{tr}(\mathbf{A}).
\end{align*}
Since $X_k=T(\mathbf{A})$, the trace estimator is unbiased:
\begin{align}
    \mathbb{E}[T(\mathbf{A})]
    =
    \operatorname{tr}(\mathbf{A}).
    \label{eq:rts_trace_unbiased}
\end{align}

\noindent \textbf{Variance.}
Using the conditional second-moment identity together with
Equation~\eqref{eq:conditional_trace_expectation} and
\eqref{eq:conditional_trace_variance}, we obtain
\begin{align*}
    \mathbb{E}[X_i^2\mid\mathcal{F}_{i-1}]
    &=
    \operatorname{Var}(X_i\mid\mathcal{F}_{i-1})
    +
    \left(
        \mathbb{E}[X_i\mid\mathcal{F}_{i-1}]
    \right)^2\\
    &\le
    \frac{c_i}{m}X_{i-1}^2+X_{i-1}^2\\
    &=
    \left(1+\frac{c_i}{m}\right)X_{i-1}^2.
\end{align*}
Taking expectations and iterating from $i=1$ to $i=k$ gives
\begin{align}
    \mathbb{E}[X_k^2]
    &\le
    \prod_{i=1}^{k}
    \left(1+\frac{c_i}{m}\right)
    X_0^2 \notag\\
    &=
    \prod_{i=1}^{k}
    \left(1+\frac{c_i}{m}\right)
    \bigl(\operatorname{tr}(\mathbf{A})\bigr)^2.
    \label{eq:composed_second_moment}
\end{align}

Combining Equation~\eqref{eq:rts_trace_unbiased} and
\eqref{eq:composed_second_moment}, we obtain
\begin{align}
    \operatorname{Var}(T(\mathbf{A}))
    &=
    \mathbb{E}[X_k^2]
    -
    \bigl(\mathbb{E}[X_k]\bigr)^2 \notag\\
    &\le
    \left[
        \prod_{i=1}^{k}
        \left(1+\frac{c_i}{m}\right)-1
    \right]
    \bigl(\operatorname{tr}(\mathbf{A})\bigr)^2.
    \label{eq:trace_variance_product_bound}
\end{align}

The sum of the layer constants is
\begin{align*}
    \sum_{i=1}^{k}c_i
    &=
    2p+8(p-1)\\
    &=
    10p-8.
\end{align*}
Using $1+x\leq e^x$ for $x\geq0$, we have
\begin{align*}
    \prod_{i=1}^{k}
    \left(1+\frac{c_i}{m}\right)
    &\le
    \exp\left(\frac{1}{m}\sum_{i=1}^{k}c_i\right)\\
    &=
    \exp\left(\frac{10p-8}{m}\right).
\end{align*}
Consequently, the general variance bound is
\begin{align}
    \operatorname{Var}(T(\mathbf{A}))
    &\le
    \left[
        \exp\left(\frac{10p-8}{m}\right)-1
    \right]
    \bigl(\operatorname{tr}(\mathbf{A})\bigr)^2.
    \label{eq:trace_variance_exponential_bound}
\end{align}

If $m\geq10p-8$, then
\begin{align*}
    0\leq\frac{10p-8}{m}\leq1.
\end{align*}
The inequality $e^x-1\leq x+x^2$, valid for $0\leq x\leq1$,
therefore gives
\begin{align*}
    \operatorname{Var}(T(\mathbf{A}))
    &\le
    \left[
        \frac{10p-8}{m}
        +
        \frac{(10p-8)^2}{m^2}
    \right]
    \bigl(\operatorname{tr}(\mathbf{A})\bigr)^2\\
    &\le
    \left(
        \frac{10p}{m}
        +
        \frac{100p^2}{m^2}
    \right)
    \bigl(\operatorname{tr}(\mathbf{A})\bigr)^2.
\end{align*}
This proves the claimed variance bound.
\end{proof}

% \begin{remark}[Comparison with Kronecker-Hutchinson estimators~\cite{meyer2025hutchinsonsestimatorbadkroneckertraceestimation} Variance Bounds]\label{remark_on_trace_and_frobenius_relation}
% The variance bound in Theorem~\ref{thm:rts_trace_estimator} is expressed in terms of the Frobenius norm $\|\mathbf{A}\|_F^2$, whereas prior work on Kronecker-Hutchinson estimators~\cite{meyer2025hutchinsonsestimatorbadkroneckertraceestimation} bounds the variance in terms of $(\operatorname{tr}(\mathbf{A}))^2$. 

% For PSD matrices, these quantities are related via
% \begin{align}
% (\operatorname{tr}(\mathbf{A}))^2 = \left(\sum_i \lambda_i\right)^2 
% \ge \sum_i \lambda_i^2 = \|\mathbf{A}\|_F^2,
% \end{align}
% where $\{\lambda_i\}\geq 0$ are the eigenvalues of $\mathbf{A}$. Thus, the Frobenius-norm-based bound is never worse and can be significantly tighter.
% % {\color{red} The above is true when $\mathbf{A}$ is PSD. Pls fix this and clarify on all instances of [13]}
% \end{remark}

\noindent We now analyze our trace estimator's randomness complexity. The following
lemma separately counts the random bits required for the leaf-level
\texttt{CountSketch} matrices in $\mathbf{T}^p$ and the internal
degree-$2$ \texttt{TensorSketch} matrices in $\mathbf{Q}^p$.
\begin{lemma}[Randomness Complexity of Recursive TensorSketch]
\label{lem:rts_randomness}
Let $\mathbf{\Pi}^{p} = \mathbf{Q}^{p} \mathbf{T}^{p}$ be the \texttt{Recursive TensorSketch} matrix as defined in Definition~\ref{def:Recursive_tensor_sketch_framework}. Then, the total number of random bits required to construct $\mathbf{\Pi}^{p}$ are $O\!\big(p (d + m)\log m\big)$.
Consequently, the estimator
$T(\mathbf{A}) := \operatorname{tr}\!\big(\mathbf{\Pi}^{p} \mathbf{A} (\mathbf{\Pi}^{p})^\top\big)$
can be implemented using
$O\!\big(p (d + m)\log m\big)$
random bits.
\end{lemma}
\begin{proof}
We decompose the randomness required to construct $\mathbf{\Pi}^p = \mathbf{Q}^p \mathbf{T}^p$ into two parts.

\noindent \textbf{(1) Randomness for $\mathbf{T}^p$.}
Recall that $\mathbf{T}^p = \mathbf{T}_1 \otimes \cdots \otimes \mathbf{T}_p$,
where each $\mathbf{T}_i \in \mathbb{R}^{m \times d}$ is a \texttt{CountSketch} matrix. Each $\mathbf{T}_i$ is specified by:
\begin{itemize}
    \item a hash function $h_i : [d] \to [m]$, requiring $\lceil \log_2 m \rceil$ bits per coordinate (to store the index $j\in [d]$ is mapped into which index  $j'\in [m]$),
    \item a sign function $s_i : [d] \to \{\pm 1\}$, requiring $1$ bit per coordinate.
\end{itemize}
Thus, each $\mathbf{T}_i$ requires
$d \big(\lceil \log_2 m \rceil + 1\big)$
random bits, and over all $p$ matrices,
\begin{align}
\text{Number of bits in $\mathbf{T}^p$}
=
p d \big(\lceil \log_2 m \rceil + 1\big).
\end{align}

% \noindent \textbf{(2) Randomness for $\mathbf{Q}^p$.}
% By definition,
% \begin{align*}
% \mathbf{Q}^p = \mathbf{S}^2 \cdot \mathbf{S}^4 \cdots \mathbf{S}^{p},
% \end{align*}
% where each $\mathbf{S}^\ell$ is a Kronecker product of $\ell/2$ matrices 
% $\mathbf{S}_j^\ell \in \mathbb{R}^{m \times m^2}$, each being a degree-$2$ \texttt{TensorSketch}.

% Each $\mathbf{S}_j^\ell$ is specified by:
% \begin{itemize}
%     \item a hash function from $[m]^2 \to [m]$, requiring $O(\log m)$ bits per coordinate, {\color{red} pls state it clearly that it obtained by two hash functions ...and refer the background}
%     \item a sign function, requiring $1$ bit per coordinate.
% \end{itemize}

\noindent \textbf{(2) Randomness for $\mathbf{Q}^p$.}
By definition,
\begin{align*}
\mathbf{Q}^p = \mathbf{S}^2 \cdot \mathbf{S}^4 \cdots \mathbf{S}^{p},
\end{align*}
where each $\mathbf{S}^\ell$ is a Kronecker product of $\ell/2$ matrices
$\mathbf{S}_j^\ell \in \mathbb{R}^{m \times m^2}$, each being a degree-$2$
\texttt{TensorSketch} as defined in Definition~\ref{def:TensorSketch of Degree Two}.
In particular, each $\mathbf{S}_j^\ell$ is constructed using two
$3$-wise independent hash functions and two $4$-wise independent random
sign functions, as specified in Definition~\ref{def:TensorSketch of Degree Two}.

Thus, each $\mathbf{S}_j^\ell$ requires:
\begin{itemize}
    \item two hash functions $h_1,h_2:[m]\to[m]$, requiring
    $O(\log m)$ bits per coordinate for each hash function, and
    \item two sign functions $\sigma_1,\sigma_2:[m]\to\{-1,+1\}$,
    requiring $1$ bit per coordinate for each sign function.
\end{itemize}

\noindent Since each $\mathbf{S}_j^\ell$ acts on $m^2$ coordinates but is implemented implicitly via hash functions, its description requires $O(m(\log m + 1))$ random bits. At level $\ell$, there are $\ell/2$ such matrices, hence
\begin{align*}
\text{Number of bits in $\mathbf{S}^\ell$}
=
O\!\big(\ell \, m (\log m + 1)\big).
\end{align*}

Summing over levels $\ell = 2,4,\dots,p$,
\begin{align}
\text{Number of bits in $\mathbf{Q}^p$}
&=
\sum_{\ell} O\!\big(\ell \, m (\log m + 1)\big) \notag
\\
&=
O\!\big(p m (\log m + 1)\big).
\end{align}

\noindent \textbf{(3) Total randomness.}
Combining both parts, we have
\begin{align}
O\!\big(p (d + m)\log m\big).
\end{align}

This proves the stated bound. The final asymptotic form follows immediately.
\end{proof}

\noindent We conclude this section analysis by deriving a concentration guarantee for the
estimator. The following theorem combines the unbiasedness and variance
bound from Theorem~\ref{thm:rts_trace_estimator} with Chebyshev's
inequality to obtain a relative failure-probability bound and a
sufficient condition on the sketch dimension $m$ for an
$(\varepsilon,\delta)$-approximation.
\begin{restatable}{theorem}{CAofTheRTS}
[Concentration Analysis of RTS Trace Estimator]
\label{concentration_analysis_RTS}
Let $\mathbf{A}\in\mathbb R^{d^p\times d^p}$ be a fixed nonzero
symmetric positive semidefinite matrix and let
$T(\mathbf{A})$ be the trace estimator defined in
Definition~\ref{def:rts_trace_estimator}. Then, for every
$\varepsilon>0$,
\begin{align}
\Pr\!\left[
    \left|T(\mathbf{A})-\operatorname{tr}(\mathbf{A})\right|
    \ge
    \varepsilon\operatorname{tr}(\mathbf{A})
\right]
\le
\frac{1}{\varepsilon^2}
\left(
    \frac{10p}{m}
    +
    \frac{100p^2}{m^2}
\right).
\label{eq:rts_relative_concentration}
\end{align}
Moreover, for $0<\varepsilon\le1$ and $\delta\in(0,1)$,
$T(\mathbf{A})$ is an $(\varepsilon,\delta)$-approximation whenever
\begin{align}
    m
    \ge
    \frac{20p}{\varepsilon^2\delta}.
\label{eq:rts_relative_sample_bound}
\end{align}
\end{restatable}

\begin{proof}
From Theorem~\ref{thm:rts_trace_estimator}, we have
\begin{align}
    \mathbb E[T(\mathbf{A})]
    &=
    \operatorname{tr}(\mathbf{A}),
    \\
    \operatorname{Var}(T(\mathbf{A}))
    &\le
    \left(
        \frac{10p}{m}
        +
        \frac{100p^2}{m^2}
    \right)
    \bigl(\operatorname{tr}(\mathbf{A})\bigr)^2.
\end{align}
Since $\mathbf{A}$ is nonzero and positive semidefinite,
$\operatorname{tr}(\mathbf{A})>0$. Therefore, Chebyshev's inequality
gives
\begin{align}
\Pr\!\left[
    \left|T(\mathbf{A})-\operatorname{tr}(\mathbf{A})\right|
    \ge
    \varepsilon\operatorname{tr}(\mathbf{A})
\right]
&\le
\frac{\operatorname{Var}(T(\mathbf{A}))}
     {\varepsilon^2
      \bigl(\operatorname{tr}(\mathbf{A})\bigr)^2}
\nonumber\le
\frac{1}{\varepsilon^2}
\left(
    \frac{10p}{m}
    +
    \frac{100p^2}{m^2}
\right).
\end{align}

\noindent To make the failure probability at most $\delta$, it is sufficient
that
\begin{align}
    \frac{1}{\varepsilon^2}
    \left(
        \frac{10p}{m}
        +
        \frac{100p^2}{m^2}
    \right)
    \le
    \delta.
\end{align}
Multiplying both sides by $m^2\varepsilon^2$ gives
\begin{align}
    \varepsilon^2\delta\,m^2
    -
    10pm
    -
    100p^2
    \ge0.
\end{align}
Solving this quadratic inequality for $m$ yields
\begin{align}
    m
    \ge
    \frac{5p}{\varepsilon^2\delta}
    \left(
        1+
        \sqrt{1+4\varepsilon^2\delta}
    \right).
\label{eq:rts_exact_relative_sample_bound}
\end{align}
Since $0<\varepsilon\le1$ and $\delta\in(0,1)$,
\(
    1+\sqrt{1+4\varepsilon^2\delta}
    \le
    1+\sqrt5
    <4.
\)
Hence,
\[
    \frac{5p}{\varepsilon^2\delta}
    \left(
        1+\sqrt{1+4\varepsilon^2\delta}
    \right)
    <
    \frac{20p}{\varepsilon^2\delta}.
\]
Therefore, the condition
\[
    m
    \ge
    \frac{20p}{\varepsilon^2\delta}
\]
is sufficient to make the failure probability at most $\delta$.
\end{proof}

%% file: complex_rts_analysis.tex
While the \texttt{Recursive TensorSketch} yields favourable variance bounds in the real-valued setting, further improvements can be obtained by considering complex-valued sketching constructions. As observed in prior work~\cite{meyer2025hutchinsonsestimatorbadkroneckertraceestimation}, complex random projections often exhibit improved concentration properties and reduced variance compared to their real-valued counterparts. Motivated by this, we extend the \texttt{Recursive TensorSketch} framework to the complex domain and analyze the resulting trace estimator in the section below.

\section{Trace Estimator using Complex Recursive TensorSketch}\label{complex_analysis_section}
\noindent In this section,
Definition~\ref{def:complex_rts_trace_estimator} introduces the Complex
\texttt{Recursive TensorSketch} trace estimator.
Using the independent-layer factorization from
Lemma~\ref{lem:factorisation}, Theorem~\ref{thm:complex_rts_trace_estimator}
establishes the unbiasedness and variance bound of the estimator.
Lemma~\ref{lem:complex_rts_randomness} then analyzes the number of
random bits required to construct the complex sketch. Finally,
Theorem~\ref{concentration_analysis_complex_RTS} derives the
corresponding concentration guarantee.

\begin{definition}[Complex Recursive TensorSketch (RTS) Trace Estimator]
\label{def:complex_rts_trace_estimator}
Let $\mathbf{\Pi}^{p}\in\mathbb{C}^{m\times d^p}$ denote the Complex
\texttt{Recursive TensorSketch} matrix constructed as in
Definition~\ref{def:Recursive_tensor_sketch_framework}, so that
$\mathbf{\Pi}^p=\mathbf{Q}^p\mathbf{T}^p$, except that the real-valued sign functions used in the
\texttt{CountSketch} and degree-$2$ \texttt{TensorSketch} matrices are
replaced by independent hash functions whose values are uniformly
distributed over the fourth roots of unity
\(
    \{1,\mathrm{i},-1,-\mathrm{i}\}.
\)
For a PSD matrix
$\mathbf{A}\in\mathbb{R}^{d^p\times d^p}$, we define
\begin{align}
    T_C(\mathbf{A})
    :=
    \operatorname{tr}\!\left(
        \mathbf{\Pi}^p\mathbf{A}(\mathbf{\Pi}^p)^*
    \right),
\end{align}
where $(\cdot)^*$ denotes the conjugate transpose.
\end{definition}

\begin{restatable}{theorem}{rtstraceestimatorcomplex}
[Unbiasedness and Variance of Complex \texttt{Recursive TensorSketch}
Trace Estimator]
\label{thm:complex_rts_trace_estimator}
Let $T_C(\mathbf{A})$ be the estimator defined in
Definition~\ref{def:complex_rts_trace_estimator}. Then
\begin{align}
    \mathbb{E}[T_C(\mathbf{A})]
    =
    \operatorname{tr}(\mathbf{A}), \quad \text{and} \quad \operatorname{Var}(T_C(\mathbf{A}))
    \leq
    \left(
        \frac{4p}{m}
        +
        \frac{16p^2}{m^2}
    \right)
    \bigl(\operatorname{tr}(\mathbf{A})\bigr)^2.
\end{align}
\end{restatable}

\begin{proof}
Let $k=2p-1$. By Lemma~\ref{lem:factorisation}, the Complex
\texttt{Recursive TensorSketch} matrix admits the independent-layer
factorization
\begin{align*}
    \mathbf{\Pi}^p
    =
    \mathbf{M}^{(k)}
    \mathbf{M}^{(k-1)}
    \cdots
    \mathbf{M}^{(1)}.
\end{align*}
Each factor is a Kronecker-wrapped sketching matrix of the form
\begin{align*}
    \mathbf{M}^{(i)}
    =
    \mathbf{I}_{d_a^{(i)}}
    \otimes
    \mathbf{K}^{(i)}
    \otimes
    \mathbf{I}_{d_b^{(i)}},
\end{align*}
where the identity matrices act on the tensor components that remain
unchanged and
\begin{align*}
    \mathbf{K}^{(i)}
    =
    \begin{cases}
        \mathbf{T}_i\in\mathbb{C}^{m\times d},
        & i=1,\ldots,p,\\[0.3em]
        \mathbf{S}_{j_i}^{\ell_i}\in\mathbb{C}^{m\times m^2},
        & i=p+1,\ldots,2p-1.
    \end{cases}
\end{align*}
Here, $\mathbf{T}_i$ is a Complex \texttt{CountSketch} matrix given in Appendix~\ref{app-real} acting
at the $i$th leaf, whereas $\mathbf{S}_{j_i}^{\ell_i}$ is a degree-$2$
Complex \texttt{TensorSketch} matrix given in Appendix~\ref{app-complex} acting at an internal node. In
particular, $\mathbf{M}^{(k)}$ contains the degree-$2$ Complex
\texttt{TensorSketch} transformation at the root node. The random
choices used in the $2p-1$ factors are mutually independent.

The proof of Lemma~\ref{lem:factorisation} depends only on the recursive
Kronecker structure and matrix multiplication. It therefore applies
over $\mathbb{C}$ after replacing the real sign functions by random hash function drawn from fourth root of unity i.i.d.

Define
\begin{align*}
    \mathbf{A}_0
    &:=
    \mathbf{A},\\
    \mathbf{A}_i
    &:=
    \mathbf{M}^{(i)}
    \mathbf{A}_{i-1}
    (\mathbf{M}^{(i)})^*,
    \qquad i=1,\ldots,k,
\end{align*}
and let
\begin{align*}
    X_i
    :=
    \operatorname{tr}(\mathbf{A}_i).
\end{align*}
Since $\mathbf{A}_0\succeq0$ and each $\mathbf{A}_i$ is obtained from
$\mathbf{A}_{i-1}$ by multiplication with $\mathbf{M}^{(i)}$ and its
conjugate transpose, positive semidefiniteness is preserved at every
step. Therefore,
\begin{align}
    \mathbf{A}_i\succeq0
    \qquad\text{for every }i=0,\ldots,k.
    \label{eq:Ai_psd_complex}
\end{align}
Consequently, each $X_i$ is real and nonnegative.

Let $\mathcal{F}_{i-1}$ represent all the random choices made in the
first $i-1$ layers. Conditional on this information,
$\mathbf{A}_{i-1}$ is fixed and positive semidefinite, while
$\mathbf{M}^{(i)}$ remains independent and random. The proof of
Lemma~\ref{lem:wrapped_layer_trace} applies over $\mathbb{C}$ after
replacing the transpose by the conjugate transpose. Hence,
\begin{align}
    \mathbb{E}[X_i\mid\mathcal{F}_{i-1}]
    &=
    X_{i-1},
    \label{eq:conditional_trace_expectation_complex}\\
    \operatorname{Var}(X_i\mid\mathcal{F}_{i-1})
    &\leq
    \frac{c_i}{m}X_{i-1}^2,
    \label{eq:conditional_trace_variance_complex}
\end{align}
where
\begin{align*}
    c_i
    =
    \begin{cases}
        1,
        & i=1,\ldots,p,\\[0.3em]
        3,
        & i=p+1,\ldots,2p-1.
    \end{cases}
\end{align*}
The constant $1$ is the second-moment JL constant for Complex
\texttt{CountSketch}, as established in
Theorem~\ref{thm:complex_countsketch_unbiased_variance} of Appendix~\ref{app-real}. The constant
$3$ is the corresponding constant for degree-$2$ Complex
\texttt{TensorSketch}, obtained from
Theorem~\ref{thm:complex_tensorsketch_degree_p} of Appendix~\ref{app-complex} by setting the degree
equal to $2$.

\noindent\textbf{Expectation.}
Taking expectations in
Equation~\eqref{eq:conditional_trace_expectation_complex} and applying
the tower property gives
\begin{align*}
    \mathbb{E}[X_i]
    =
    \mathbb{E}[X_{i-1}].
\end{align*}
Iterating over all $k$ layers yields
\begin{align*}
    \mathbb{E}[X_k]
    =
    \mathbb{E}[X_0]
    =
    \operatorname{tr}(\mathbf{A}).
\end{align*}
Since $X_k=T_C(\mathbf{A})$, it follows that
\begin{align}
    \mathbb{E}[T_C(\mathbf{A})]
    =
    \operatorname{tr}(\mathbf{A}).
    \label{eq:complex_rts_unbiased}
\end{align}

\noindent\textbf{Variance.}
Using the conditional second-moment identity together with
Equations~\eqref{eq:conditional_trace_expectation_complex} and
\eqref{eq:conditional_trace_variance_complex}, we obtain
\begin{align*}
    \mathbb{E}[X_i^2\mid\mathcal{F}_{i-1}]
    &=
    \operatorname{Var}(X_i\mid\mathcal{F}_{i-1})
    +
    \left(
        \mathbb{E}[X_i\mid\mathcal{F}_{i-1}]
    \right)^2\\
    &\leq
    \frac{c_i}{m}X_{i-1}^2
    +
    X_{i-1}^2\\
    &=
    \left(
        1+\frac{c_i}{m}
    \right)X_{i-1}^2.
\end{align*}
Taking expectations and iterating gives
\begin{align}
    \mathbb{E}[X_k^2]
    &\leq
    \prod_{i=1}^{k}
    \left(
        1+\frac{c_i}{m}
    \right)
    X_0^2 \notag\\
    &=
    \prod_{i=1}^{k}
    \left(
        1+\frac{c_i}{m}
    \right)
    \bigl(\operatorname{tr}(\mathbf{A})\bigr)^2.
    \label{eq:complex_composed_second_moment}
\end{align}

Combining Equations~\eqref{eq:complex_rts_unbiased} and
\eqref{eq:complex_composed_second_moment}, we obtain
\begin{align}
    \operatorname{Var}(T_C(\mathbf{A}))
    &=
    \mathbb{E}[X_k^2]
    -
    \bigl(\mathbb{E}[X_k]\bigr)^2 \notag\\
    &\leq
    \left[
        \prod_{i=1}^{k}
        \left(
            1+\frac{c_i}{m}
        \right)
        -1
    \right]
    \bigl(\operatorname{tr}(\mathbf{A})\bigr)^2.
    \label{eq:complex_variance_product}
\end{align}

Using $e^x-1\leq x+x^2$ for $0\leq x\leq1$, we obtain
\begin{align*}
    \operatorname{Var}(T_C(\mathbf{A}))
    &\leq
    \left[
        \frac{4p-3}{m}
        +
        \frac{(4p-3)^2}{m^2}
    \right]
    \bigl(\operatorname{tr}(\mathbf{A})\bigr)^2\\
    &\leq
    \left(
        \frac{4p}{m}
        +
        \frac{16p^2}{m^2}
    \right)
    \bigl(\operatorname{tr}(\mathbf{A})\bigr)^2.
\end{align*}
This proves the stated unbiasedness and variance bounds.
\end{proof}

\begin{lemma}[Randomness Complexity of Complex Recursive TensorSketch]
\label{lem:complex_rts_randomness}
Let $\mathbf{\Pi}^{p} = \mathbf{Q}^{p} \mathbf{T}^{p}$ be the complex \texttt{Recursive TensorSketch} matrix. 
Then, the total number of random bits required to construct $\mathbf{\Pi}^{p}$ is 
$O\!\big(p (d + m)\log m\big)$. Consequently, the estimator
\(
T_C(\mathbf{A}) := \operatorname{tr}\!\big(\mathbf{\Pi}^{p} \mathbf{A} (\mathbf{\Pi}^{p})^*\big)
\)
can be implemented using 
$O\!\big(p (d + m)\log m\big)$ random bits.
\end{lemma}

\begin{proof}
The proof follows along the structure as in Lemma~\ref{lem:rts_randomness}. 
In the complex setting, each random variable can be expressed in the form $a + ib$, 
where $a$ and $b$ are real-valued random variables. Thus, compared to the real case, 
the construction involves at most a constant factor increase in the number of underlying 
random variables.
% \noindent Since we are working with asymptotic bounds, this constant factor does not affect the overall complexity. 
Therefore, the total number of random bits required remains 
$O\!\big(p (d + m)\log m\big)$.
\end{proof}

% \begin{restatable}{theorem}{CAofTheCRTS}[Concentration Analysis of Complex RTS Trace Estimator] 
% % \begin{theorem}[Concentration Analysis of Complex RTS Trace Estimator]
% \label{concentration_analysis_complex_RTS}
% Let $\mathbf{A} \in \mathbb{R}^{d^p \times d^p}$ be a fixed matrix, and let 
% $T_C(\mathbf{A})$ denote the trace estimator defined in Definition~\ref{def:complex_rts_trace_estimator}. 
% Then for any $\varepsilon > 0$ and $\delta \in (0,1)$,
% \begin{equation}
% \Pr\!\left[\, |T_C(\mathbf{A}) - \operatorname{tr}(\mathbf{A})| \ge \varepsilon \|\mathbf{A}\|_F \,\right]
% \;\le\; \frac{p}{m \varepsilon^2}.
% \end{equation}
% \end{restatable}

% \begin{proof}
%     From Theorem~\ref{thm:complex_rts_trace_estimator}, we have
% \[
% \mathbb{E}[T_C(\mathbf{A})] = \operatorname{tr}(\mathbf{A}), \quad \text{and} \quad 
% \operatorname{Var}(T_C(\mathbf{A})) \le \left(\frac{4p}{m} + \frac{16p^2}{m^2}\right)\bigl(\operatorname{tr}(\mathbf{A})\bigr)^2.
% \]

% Applying Chebyshev's inequality,
% \[
% \Pr\!\left[\, |T_C(\mathbf{A}) - \mathbb{E}[T_C(\mathbf{A})]| \ge \varepsilon \|\mathbf{A}\|_F \,\right]
% \;\le\; \frac{\operatorname{Var}(T_C(\mathbf{A}))}{\varepsilon^2 \|\mathbf{A}\|_F^2}.
% \]

% \noindent Substituting the variance bound from Theorem~\ref{thm:complex_rts_trace_estimator} gives
% \[
% \Pr\!\left[\, |T_C(\mathbf{A}) - \operatorname{tr}(\mathbf{A})| \ge \varepsilon \|\mathbf{A}\|_F \,\right]
% \;\le\; \frac{p}{m \varepsilon^2}.
% \]

% \noindent Setting the right-hand side to be at most $\delta$ yields
% \(
% m = O\!\left(\frac{p}{\varepsilon^2 \delta}\right).
% \)
% \end{proof}

\begin{restatable}{theorem}{CAofTheCRTS}
[Concentration Analysis of Complex RTS Trace Estimator]
\label{concentration_analysis_complex_RTS}
Let $\mathbf{A}\in\mathbb R^{d^p\times d^p}$ be a fixed nonzero
symmetric positive semidefinite matrix and let
$T_C(\mathbf{A})$ be the trace estimator defined in
Definition~\ref{def:complex_rts_trace_estimator}. Then, for every
$\varepsilon>0$,
\begin{align}
\Pr\!\left[
    \left|T_C(\mathbf{A})-\operatorname{tr}(\mathbf{A})\right|
    \ge
    \varepsilon\operatorname{tr}(\mathbf{A})
\right]
\leq
\frac{1}{\varepsilon^2}
\left(
    \frac{4p}{m}
    +
    \frac{16p^2}{m^2}
\right).
\label{eq:complex_rts_relative_concentration}
\end{align}
Moreover, for $0<\varepsilon\leq1$ and $\delta\in(0,1)$,
$T_C(\mathbf{A})$ is an $(\varepsilon,\delta)$-approximation whenever
\begin{align}
    m
    \geq
    \frac{8p}{\varepsilon^2\delta}.
\label{eq:complex_rts_relative_sample_bound}
\end{align}
\end{restatable}

\begin{proof}
From Theorem~\ref{thm:complex_rts_trace_estimator}, we have
\begin{align}
    \mathbb{E}[T_C(\mathbf{A})]
    &=
    \operatorname{tr}(\mathbf{A}),
    \\
    \operatorname{Var}(T_C(\mathbf{A}))
    &\leq
    \left(
        \frac{4p}{m}
        +
        \frac{16p^2}{m^2}
    \right)
    \bigl(\operatorname{tr}(\mathbf{A})\bigr)^2.
\end{align}
Since $\mathbf{A}$ is nonzero and positive semidefinite,
$\operatorname{tr}(\mathbf{A})>0$. Therefore, Chebyshev's inequality
gives
\begin{align}
\Pr\!\left[
    \left|T_C(\mathbf{A})-\operatorname{tr}(\mathbf{A})\right|
    \geq
    \varepsilon\operatorname{tr}(\mathbf{A})
\right]
&\leq
\frac{\operatorname{Var}(T_C(\mathbf{A}))}
     {\varepsilon^2
      \bigl(\operatorname{tr}(\mathbf{A})\bigr)^2}
\nonumber\\
&\leq
\frac{1}{\varepsilon^2}
\left(
    \frac{4p}{m}
    +
    \frac{16p^2}{m^2}
\right).
\end{align}

\noindent
To make the failure probability at most $\delta$, it is sufficient
that
\begin{align}
    \frac{1}{\varepsilon^2}
    \left(
        \frac{4p}{m}
        +
        \frac{16p^2}{m^2}
    \right)
    \leq
    \delta.
\end{align}
Multiplying both sides by $m^2\varepsilon^2$ gives
\begin{align}
    \varepsilon^2\delta\,m^2
    -
    4pm
    -
    16p^2
    \geq0.
\end{align}
Solving this quadratic inequality for $m$ yields
\begin{align}
    m
    \geq
    \frac{2p}{\varepsilon^2\delta}
    \left(
        1+
        \sqrt{1+4\varepsilon^2\delta}
    \right).
\label{eq:complex_rts_exact_relative_sample_bound}
\end{align}
Since $0<\varepsilon\leq1$ and $\delta\in(0,1)$,
\(
    1+\sqrt{1+4\varepsilon^2\delta}
    \leq
    1+\sqrt5
    <4.
\)
Hence,
\[
    \frac{2p}{\varepsilon^2\delta}
    \left(
        1+\sqrt{1+4\varepsilon^2\delta}
    \right)
    <
    \frac{8p}{\varepsilon^2\delta}.
\]
Therefore, the condition
\[
    m
    \geq
    \frac{8p}{\varepsilon^2\delta}
\]
is sufficient to make the failure probability at most $\delta$.
\end{proof}

%% file: conclusion.tex
\section{Conclusion}
In this paper, we introduce a trace estimation algorithm for an implicit matrix $\mathbf{A} \in \mathbb{R}^{d^p \times d^p}$ based on \texttt{Recursive TensorSketch} $\Pi^{p} \in \mathbb{R}^{m \times d^p}$ proposed in ~\cite{doi:10.1137/1.9781611975994.9}. Our estimator leverages structured random projections and requires significantly fewer random bits than existing baselines, while maintaining strong theoretical guarantees.
We show that the proposed estimator is unbiased and admits a variance bound of $O\!\left(\left(\frac{10p}{m} + \frac{100p^2}{m^2}\right)\bigl(\operatorname{tr}(\mathbf{A})\bigr)^2\right)$. 
We further introduce a complex-valued variant, in which the entries of $\Pi^{p}$ are sampled from complex random variables, leading to improved variance bounds. 
In contrast to the Kronecker-Hutchinson estimator of~\cite{meyer2025hutchinsonsestimatorbadkroneckertraceestimation}, our estimator avoids exponential dependence on $p$ in variance bounds of the respective estimators.
These properties make the proposed approach well-suited for high-dimensional settings. 
Our work also suggests several directions for future investigation.

Several improved variants of the Hutchinson trace estimator have been proposed that offer additional variance reduction, such as \textit{Hutch++}~\cite{DBLP:conf/sosa/MeyerMMW21, DBLP:journals/siammax/PerssonCK22} and Krylov-aware trace estimation~\cite{chen2023krylov}. It would be interesting to investigate whether our approach can be combined with these techniques to achieve further variance reduction. Extensions of the Hutchinson trace estimator have also been developed for related problems, including the estimation of $\operatorname{tr}\left(f(\mathbf{A})\right)$~\cite{DBLP:journals/siammax/Ubaru0S17}, multivariate trace estimation~\cite{DBLP:journals/siammax/MorYosefUHA25}, partial trace estimation~\cite{doi:10.1137/23M1620399}, and trace estimation for tensor data~\cite{DBLP:journals/corr/abs-2510-22157}. It would be of interest to explore whether our technique can be integrated into these frameworks to yield randomness-efficient estimators for the corresponding problems.
% In this paper, we introduced a novel trace estimation framework based on \texttt{Recursive TensorSketch}.
% Our estimator leverages structured random projections to generate query vectors with significantly
% reduced randomness, while maintaining strong theoretical guarantees. We showed that the proposed
% estimator is unbiased and achieves variance bounds of $O\!\left(\frac{3}{m}\|\mathbf{A}\|_F^2\right)$ where $\mathbf{A} \in \mathbb{R}^{d^p \times d^p}$ in the
% real setting and improved bounds in the complex setting. 

Finally, to obtain an $(\varepsilon, \delta)$-approximation, our current analysis relies on Chebyshev’s inequality, resulting in a sample complexity with suboptimal dependence on $\delta$. An important direction for future work is to improve this dependence by leveraging higher-moment analysis of the estimator.

% Establishing sharper concentration bounds could further enhance the efficiency
% of the proposed method and broaden its applicability in large-scale settings.

%% file: appendix.tex
\newpage
\section*{Appendix}

\section{Analysis of Complex Sketches}
This section provides the theoretical analysis of the complex-valued sketching constructions used in this work. We first analyze Complex \texttt{CountSketch}, deriving its unbiasedness, variance, and sketching-time guarantees. We then analyse Complex \texttt{TensorSketch} for degree-\(p\) polynomial kernels and establish the corresponding expectation, variance, and computational bounds. Finally, we prove an auxiliary complex AMS moment result used in the analysis of Complex \texttt{TensorSketch}.

\subsection{Theoretical Analysis of Complex \texttt{CountSketch}}\label{app-real}
\begin{theorem}[Unbiasedness, Variance, and Sketching Time of Complex \texttt{CountSketch} Inner-Product Estimator]
\label{thm:complex_countsketch_unbiased_variance}
Let $\mathbf{x},\mathbf{y}\in\mathbb{R}^d$, and let
$\mathbf{C}\in\mathbb{C}^{D\times d}$ be a Complex
\texttt{CountSketch} matrix. Define the inner-product estimator by $\widehat{k}_{C}(\mathbf{x},\mathbf{y})
=
\left\langle
\mathbf{C}\mathbf{x},
\mathbf{C}\mathbf{y}
\right\rangle_{\mathbb{C}},$
where $\langle\mathbf{a},\mathbf{b}\rangle_{\mathbb{C}}
=\mathbf{a}^{*}\mathbf{b}$ denotes the Hermitian inner product.
Then
\begin{align}
\mathbb{E}\!\left[
\widehat{k}_{C}(\mathbf{x},\mathbf{y})
\right]
&=
\langle\mathbf{x},\mathbf{y}\rangle,
\\
\operatorname{Var}\!\left[
\widehat{k}_{C}(\mathbf{x},\mathbf{y})
\right]
&=
\frac{1}{D}
\left(
\|\mathbf{x}\|_2^2\|\mathbf{y}\|_2^2
-
\sum_{i=1}^{d}x_i^2y_i^2
\right).
\end{align}
Moreover, the sketches $\mathbf{C}\mathbf{x}$ and $\mathbf{C}\mathbf{y}$
can be computed in
$O(\operatorname{nnz}(\mathbf{x}))$ and
$O(\operatorname{nnz}(\mathbf{y}))$ time, respectively.
\end{theorem}

% \noindent \textbf{Proof outline.} 
\begin{proof}
We first outline the structure of the proof.
To prove unbiasedness, we begin by expanding the inner product expression of the sketched vectors obtained from Complex \texttt{CountSketch}. We compute the expectation using the independence of the hash function and complex random hash funtions along with the moment properties $\mathbb{E}[s(i)\overline{s(r)}]=0$ for $i\neq r$ and $\mathbb{E}[|s(i)|^2]=1$ due to which all cross terms vanish and we obtain the unbiased estimation of actual inner product.
 For the variance, we expand the second moment of the estimator and analyze the non-zero terms. Since $\mathbb{E}[s(i)^2]=\mathbb{E}[\overline{s(i)}^2]=0$ and $\mathbb{E}[s(i)\overline{s(i)}]=\mathbb{E}[|s(i)|^2]=1$ for all $i\in [d]$ imply that all terms vanish except those corresponding to index configurations with pairwise matchings. Combining these contributions provides a closed-form expression for the second moment, and subtracting the squared mean gives a variance of order $1/D$, completing the proof.

\smallskip
We now provide the detailed argument. By expanding the estimator, we obtain
\begin{align}
    \hat{k}_C(\mathbf{x},\mathbf{y})&=\Phi_C(\mathbf{x})^{*}\overline{\Phi_C(\mathbf{y})}=\langle \mathbf{C x}, \overline{\mathbf{C y}} \rangle,\\
&=\sum_{j=1}^{D} (Cx)_j\,\overline{({Cy})_j},\notag\\
&=\sum_{j=1}^{D} (\sum_{i=1}^d s(i)\mathbf{1}_{h(i) = j}x_{i})(\sum_{r=1}^d \overline{s(r)}\mathbf{1}_{h(r) = j}y_{r}),\notag\\
&=\sum_{j=1}^{D}\;\sum_{i=1}^d\sum_{r=1}^d s(i)
\overline{s(r)}\,\mathbf{1}_{h(i) = j}\,\mathbf{1}_{h(r) = j}\,x_i\,y_r.\label{eq:expansion of countsketch estimator}
\end{align}

\noindent \textbf{Computing Expectation:}\\

We compute the expected value of Equation~\eqref{eq:expansion of countsketch estimator}.
\begin{align}
\mathbb{E}\!\left[\hat{k}_C(\mathbf{x},\mathbf{y})\right]
&=
\sum_{j=1}^{D}\sum_{i,r=1}^{d}
x_i y_r \,
\mathbb{E}\!\left[s(i)\overline{s(r)}\right]\,
\mathbb{E}\!\left[\mathbf{1}_{h(i) = j} \mathbf{1}_{h(r) = j}\right], \notag \\
&=
\sum_{j=1}^{D}\sum_{i=1}^{d}
x_i y_i \,
\mathbb{E}\!\left[|s(i)|^{2}\right]\,
\mathbb{E}\!\left[\mathbf{1}_{h(i) = j}^{2}\right]
\;+\; \cdots \notag \\
&\cdots + \sum_{j=1}^{D}\sum_{\substack{i,r=1 \\ i\neq r}}^{d}
x_i y_r \,
\mathbb{E}\!\left[s(i)\overline{s(r)}\right]\,
\mathbb{E}\!\left[\mathbf{1}_{h(i) = j} \mathbf{1}_{h(r) = j}\right].\label{eq:split expectation of kernel estimator}
\end{align}
By independence and symmetry of the  functions $h(.)$ and $s(.)$, we have
$\mathbb{E}[|s(i)|^{2}] = 1$ and
$\mathbb{E}[s(i)\overline{s(r)}] = 0$ for $i \neq r$.
Moreover, since $\mathbf{1}_{h(i) = j}^{2} = \mathbf{1}_{\{h(i)=j\}}$ with $h(i)$ uniform on $[D]$,
\[
\mathbb{E}[\mathbf{1}_{h(i) = j}^{2}] = \mathbb{E}[\mathbf{1}_{h(i) = j}] = \frac{1}{D}.
\]
Substituting these identities into
\eqref{eq:split expectation of kernel estimator} vanishes cross term and we get
\begin{align}
\mathbb{E}\!\left[\hat{k}_C(\mathbf{x},\mathbf{y})\right]
&=
\sum_{j=1}^{D} \frac{1}{D} \sum_{i=1}^{d} x_i y_i
=
\langle \mathbf{x}, \mathbf{y} \rangle. \label{eq:unbiasness complex count sketch}
\end{align}
This completes the proof of unbiasedness. We next turn to the analysis of the variance of the estimator.

\noindent \textbf{Computing Variance:}\\
The variance of the complex estimator can be expressed as
\begin{align}
\operatorname{Var}\!\left[\hat{k}_C(\mathbf{x},\mathbf{y})\right]
\;=\;
\mathbb{E}\!\left[\,\left|\hat{k}_C(\mathbf{x},\mathbf{y})\right|^{2}\,\right]
\;-\;
\left|\mathbb{E}\!\left[\hat{k}_C(\mathbf{x},\mathbf{y})\right]\right|^{2}. \label{eq:variance expression of estimator}
\end{align}
To evaluate the first term, we expand it using Equation~\eqref{eq:expansion of countsketch estimator} as follows

\begin{align}
\left|\hat{k}_C(\mathbf{x},\mathbf{y}) \right|^2
&=
\left|\langle \mathbf{C}\mathbf{x}, \overline{\mathbf{C}\mathbf{y}}\rangle\right|^2
= \Bigg|\sum_{j=1}^{D} \sum_{i,r=1}^{d}
s(i)\,\overline{s(r)}\,\mathbf{1}_{h(i) = j}\mathbf{1}_{h(r) = j}\,x_i y_r\Bigg|^2 \notag \\
&=
\sum_{j,j'=1}^{D} \sum_{i,r,p,q=1}^{d}\!\!\!\!s(i)\,\overline{s(r)}\,\overline{s(p)}\,s(q)\;
\mathbf{1}_{h(i) = j}\mathbf{1}_{h(r) = j}\mathbf{1}_{h(p) = j'}\mathbf{1}_{h(q) = j'}
x_i y_r x_p y_q .
\end{align}
Taking expectations with respect to the randomness in $h(.)$ and $s(.)$, we obtain
\begin{align}
&\mathbb{E}\!\left[\left|\hat{k}_C(\mathbf{x},\mathbf{y}) \right|^2\right]\notag\\
&\quad=
\sum_{j,j'=1}^{D}\sum_{i,r,p,q=1}^{d}\!\!\!\!x_i y_r x_p y_q
\mathbb{E}\!\left[s(i)\overline{s(r)}\overline{s(p)}s(q)\right]
\mathbb{E}\!\left[\mathbf{1}_{h(i) = j}\mathbf{1}_{h(r) = j}\mathbf{1}_{h(p) = j'}\mathbf{1}_{h(q) = j'}\right].
\label{eq:expectation on second moment of complex countsketch}
\end{align}

\noindent We begin with the \textbf{case $j = j'$}:

Since the random variables $\{s(i)\}_{i=1}^d$ are i.i.d.\ with
$\mathbb{E}[s(i)] = 0$, $\mathbb{E}[|s(i)|^2] = 1$, and $\mathbb{E}[s(i)^2] = 0$,
the fourth-order moment
\[
\mathbb{E}\!\left[s(i)\,\overline{s(r)}\,\overline{s(p)}\,s(q)\right]
\]
is nonzero only when each index appears an even number of times.
Following terms which are non-zero:

\begin{itemize}
\item[(a)] $i = r = p = q : \quad\mathbb{E}[|s(i)|^4] = 1,
\qquad \qquad \ \
\mathbb{E}[\mathbf{1}_{h(i) = j}^4] = \mathbb{E}[\mathbf{1}_{h(i) = j}] = \frac{1}{D}.$

\item[(b)] $i = r \neq p = q: \quad \mathbb{E}[|s(i)|^2 |s(p)|^2] = 1,
\qquad
\mathbb{E}[\mathbf{1}_{h(i) = j}^2 \mathbf{1}_{h(p) = j}^2] = \frac{1}{D^2}.$

\item[(c)] $i = p \neq r = q: \quad\mathbb{E}[|s(i)|^2 |s(r)|^2] = 1,
\qquad
\mathbb{E}[\mathbf{1}_{h(i) = j}^2 \mathbf{1}_{h(r) = j}^2] = \frac{1}{D^2}.$
\end{itemize}

All other cases are zero. Adding the contributions from the above cases, we obtain

\begin{align}
\label{var_comp_diagonal}
\sum_{i=1}^{d} x_i^{2} y_i^{2}
\;+\;
\frac{1}{D}
\left(
\sum_{\substack{i,p=1 \\ i \neq p}}^{d} x_i y_i x_p y_p
\;+\;
\sum_{\substack{i,r=1 \\ i \neq r}}^{d} x_i^{2} y_r^{2}
\right).
\end{align}

\noindent We next consider the \textbf{case $j \neq j'$}:
Since the same index cannot hash to two different buckets, all terms vanish except the following case.
\begin{itemize}
\item[(a)] $i=r\neq p=q$: $\;\mathbb{E}[|s(i)|^2|s(p)|^2] = 1, \qquad \mathbb{E}[\mathbf{1}_{h(i) = j}^2\mathbf{1}_{h(p) = j'}^2]=\frac{1}{D^2}.$
\end{itemize}
Therefore we have,
\begin{align}\label{var_comp_non_diagonal}
    \frac{D-1}{D}\sum_{i \neq p} x_iy_ix_py_p.
\end{align}
Combining Equation~\eqref{var_comp_diagonal} and Equation~\eqref{var_comp_non_diagonal}, we get
\begin{align}
\mathbb{E}\!\left[\left|\hat{k}_C(\mathbf{x},\mathbf{y}) \right|^2\right] &= \sum_i^d x_i^2y_i^2\! +\! \frac{1}{D} \left( \sum_{i \neq p}x_iy_ix_py_p + \sum_{i \neq r} x_i^2 y_r^2  \right) \!+\!  \frac{D-1}{D}\left(\sum_{i \neq p} x_iy_ix_py_p \right),\\
&= \langle\mathbf{x},\mathbf{y}\rangle^2 + \frac{1}{D} \sum_{i \neq r} x_i^2 y_r^2 . \label{second_moment_complex_cs}
\end{align}

\noindent Substituting Equation~\eqref{second_moment_complex_cs} and
Equation~\eqref{eq:unbiasness complex count sketch} into
Equation~\eqref{eq:variance expression of estimator}, we get

\begin{align}
\label{var_cs_final}
    \operatorname{Var}\left[\hat{k}_C(\mathbf{x},\mathbf{y})\right] &= \langle\mathbf{x},\mathbf{y}\rangle^2 + \frac{1}{D} \left( \sum_{i \neq r} x_i^2 y_r^2  \right) - \langle\mathbf{x},\mathbf{y}\rangle^2,\\
&=\frac{1}{D}\Big(\| \mathbf{x}\|_2^2\|\mathbf{y}\|_2^2-\sum_i x_i^2 y_i^2\Big).
\end{align}

\end{proof}

\begin{remark}[Sketching time for Complex \texttt{CountSketch}]
For a vector $\mathbf{x} \in \mathbb{R}^d$, the Complex \texttt{CountSketch} sketch
$\mathbf{C}\mathbf{x}$ can be computed in $O(\operatorname{nnz}(\mathbf{x}))$ time.
This is because each nonzero entry $x_i $ contributes to exactly one bucket
$h(i)$ with a single multiplication by the corresponding complex random variable $s(i)$
and a single addition, while zero entries require no computation.
\end{remark}

\subsection{Theoretical Analysis of Complex \texttt{TensorSketch}} \label{app-complex}
\begin{theorem}[Unbiasedness and Variance of Complex \texttt{TensorSketch} for Degree-$p$ Polynomial Kernel]
\label{thm:complex_tensorsketch_degree_p}
Let $\mathbf{x},\mathbf{y}\in\mathbb{R}^{d}$ and let
$\mathbf{x}^{\otimes p},\mathbf{y}^{\otimes p}\in\mathbb{R}^{d^p}$.
Let $\mathbf{C}\in\mathbb{C}^{D\times d^p}$ denote a Complex
\texttt{TensorSketch} matrix. Define the degree-$p$ polynomial kernel
estimator by $\widehat{k}_{C}(\mathbf{x},\mathbf{y})
=
\left\langle
\mathbf{C}\mathbf{x}^{\otimes p},
\mathbf{C}\mathbf{y}^{\otimes p}
\right\rangle_{\mathbb{C}}$,
where $\langle\mathbf{a},\mathbf{b}\rangle_{\mathbb{C}}
=\mathbf{a}^{*}\mathbf{b}$ denotes the Hermitian inner product.
Then
\begin{align}
\mathbb{E}\!\left[
\widehat{k}_{C}(\mathbf{x},\mathbf{y})
\right]
&=
\left\langle
\mathbf{x}^{\otimes p},
\mathbf{y}^{\otimes p}
\right\rangle
=
\langle\mathbf{x},\mathbf{y}\rangle^{p},
\\[0.5em]
\operatorname{Var}\!\left[
\widehat{k}_{C}(\mathbf{x},\mathbf{y})
\right]
&\leq
\frac{1}{D}
\left[
\left(
\|\mathbf{x}\|_2^2\|\mathbf{y}\|_2^2
-
\sum_{i=1}^{d}x_i^2y_i^2
\right)^{p}
-
\langle\mathbf{x},\mathbf{y}\rangle^{2p}
\right]
\\
&\leq
\frac{2^{p}-1}{D}
\|\mathbf{x}\|_2^{2p}
\|\mathbf{y}\|_2^{2p}.
\end{align}

Moreover, the sketches $\mathbf{C}\mathbf{x}^{\otimes p}$ and
$\mathbf{C}\mathbf{y}^{\otimes p}$ can be computed in $O\!\left(
p\left(\operatorname{nnz}(\mathbf{x})+D\log D\right)
\right)
$ and $
O\!\left(
p\left(\operatorname{nnz}(\mathbf{y})+D\log D\right)
\right)$
time, respectively.
\end{theorem}

\begin{proof}
We first outline the structure of the proof.
Complex \texttt{TensorSketch} is viewed as a Complex \texttt{CountSketch} applied to the $p$-fold tensor products
$\mathbf{x}^{\otimes p}$ and $\mathbf{y}^{\otimes p}$ via suitably defined composite hash and complex random functions.
We prove Unbiasedness by expanding the sketched inner product and using properties of the expected value of the random function $s(.)$ to eliminate all cross terms.
To analyze the variance, we expand the second moment of the estimator and 
using the independence between the functions $(H,S)$, the second moment reduces to a scaled second-moment expression involving only the random function $s(.)$.
This expression is bounded using a complex AMS moment bound, proved later in Lemma~\ref{lem:complex AMS}. Finally, the variance bound is simplified using the Cauchy-Schwarz inequality, resulting in an $O(1/D)$ bound.

\smallskip
\noindent We now present the detailed proof. We begin by noting that the \texttt{TensorSketches} $\mathrm{C}\mathbf{x}^{\otimes p}, \mathrm{C}\mathbf{y}^{\otimes p}$ are the
\texttt{CountSketches} of the tensor product 
\(
X := \mathbf{x}^{\otimes p},\ 
Y := \mathbf{y}^{\otimes p}
\)
using the two aggregated functions 
\(H : [d]^p \mapsto [D]\) 
and 
$S : [d]^p \to \{1, \omega, \omega^{2}, \omega^{3}\}$
such that:
\begin{align}
H(i_1,\ldots,i_p) &= \left( \sum_{j=1}^{p} h_j(i_j) \right) \bmod D, \\
S(i_1,\ldots,i_p) &= \prod_{j=1}^{p} s_j(i_j).
\end{align}
Also note that \( H(.) \) is \(2\)-wise independent \cite{patracscu2012power}.

\noindent For further proof, we use $u, v \in [d]^p$ as the
indices of vectors $X, Y$ of dimension $d^p$.   
Then we expand $\hat{k}_{C}(\mathbf{x}, \mathbf{y})$ as,
\begin{align}
\hat{k}_{C}(\mathbf{x}, \mathbf{y}) = \langle \mathbf{C}X, \overline{\mathbf{C}Y} \rangle
  &= \sum_{u,v \in [d]^p}
      X_{u}\, Y_{v}\, 
      S(u)\, \overline{S(v)}\, \mathbf{1}_{[ H(u) = H(v) ]}, \\
  &= \langle X,Y \rangle
   \;+\;
     \sum_{u \ne v}
      X_{u}\, Y_{v}\,
      S(u)\, \overline{S(v)}\, \mathbf{1}_{[ H(u) = H(v) ]}.
\end{align}

\noindent As we know, $\mathbb{E}[S(u)\, \overline{S(v)}] = 0, \forall \ u \neq v.$ Then we have
\begin{align}
\mathbb{E}\!\left[ 
\hat{k}_{C}(\mathbf{x}, \mathbf{y})
\right]
= \langle X, Y \rangle
= \langle \mathbf{x}, \mathbf{y} \rangle^{p}.
\end{align}

\noindent For the variance, we first compute
\(
\mathbb{E}\!\left[|
  \hat{k}_{C}(\mathbf{x}, \mathbf{y})|^{2}
\right]
\)
. Let's first expand the second moment term,  
\begin{align}
&| \hat{k}_{C}(\mathbf{x}, \mathbf{y})|^{2}
= \langle \mathbf{C}\mathbf{x}^{\otimes p}, \overline{\mathbf{C}\mathbf{y}^{\otimes p}} \rangle\langle \overline{\mathbf{C}\mathbf{x}^{\otimes p}}, \mathbf{C}\mathbf{y}^{\otimes p} \rangle \\
=&\!\! \left(\!\!\!
    \langle X,\! Y \rangle\!
    +\!\!\! \sum_{u \ne v}\!
        X_{u} Y_{v} 
        S(u) \overline{S(v)}
        \mathbf{1}_{[ H(u) = H(v) ]}\!\!
  \right) \!\!\!\left(\!\!\!
    \langle X,\! Y \rangle\!
    +\!\!\! \sum_{u \ne v}\!
        X_{u} Y_{v} 
        \overline{S(u)} S(v)
        \mathbf{1}_{[ H(u) = H(v) ]}\!\!
  \right)\!\!,\\
=& \langle X, Y \rangle^{2}
+  \langle X, Y \rangle 
    \left(\sum_{u \ne v}
        X_{u} Y_{v} 
        S(u) \overline{S(v)} 
        \mathbf{1}_{[ H(u) = H(v) ]}\right. \notag\\
&\qquad\left.{}+ \sum_{u \ne v}
        X_{u} Y_{v} 
        \overline{S(u)} S(v) 
        \mathbf{1}_{[ H(u) = H(v) ]}\right) \notag\\
&\qquad{}+
\left|\left(
\sum_{u \ne v}
    X_{u} Y_{v} 
    S(u) \overline{S(v)} 
    \mathbf{1}_{[ H(u) = H(v) ]}
\right)\right|^{2}.
\end{align}
Now, take the expectation of $| \hat{k}_{C}(\mathbf{x}, \mathbf{y})|^{2}$ and we know that $\mathbb{E}\left[\overline{S(u)} S(v)\right] = \mathbb{E}\left[S(u) \overline{S(v)}\right] =0 , \forall  u\neq v.$ Then,
\begin{align}
      \mathbb{E}\left[  | \langle \mathbf{C}\mathbf{x}^{\otimes p}, \overline{\mathbf{C}\mathbf{y}^{\otimes p}} \rangle|^{2}\right] = \left\langle X,Y \right\rangle^{2} +    \mathbb{E}\!\left[
\left|\left(
\sum_{u \ne v}
    X_{u} Y_{v} 
    S(u) \overline{S(v)} 
    \mathbf{1}_{[ H(u) = H(v) ]}
\right)\right|^{2}
\right]. \label{eq: second moment for tensor sketch to be used for variance}
\end{align}
Using the fact that functions $S$ and $H$ are independent and Lemma~\ref{lem:complex AMS} (proved below), we can bound the expectation of the second non-diagonal term in the above equation. 
\begin{align}
\mathbb{E}\Bigg[
\Bigg|\Bigg(
\sum_{u \ne v}
    X_{u} Y_{v}& 
    S(u) \overline{S(v)} 
    \mathbf{1}_{[ H(u) = H(v) ]}
\Bigg)\Bigg|^{2}
\Bigg] =
\mathbb{E}\!\Bigg[
\sum_{\substack{u_1 \ne v_1 \\ u_2 \ne v_2}}
    X_{u_1} Y_{v_1}
    X_{u_2} Y_{v_2} \times \cdots \notag \\
    &\cdots \times S(u_1) \overline{S(v_1)}
    \overline{S(u_2)} S(v_2)
    \mathbf{1}_{[ H(u_1) = H(v_1) ]}
    \mathbf{1}_{[ H(u_2) = H(v_2) ]}
\Bigg],
\end{align}
    
\begin{align}
&=
\sum_{\substack{u_1 \ne v_1 \\ u_2 \ne v_2}}
\mathbb{E}\!\left[
    X_{u_1} Y_{v_1}
    X_{u_2} Y_{v_2}
    S(u_1) \overline{S(v_1)}
    \overline{S(u_2)} S(v_2)
\right]
\cdot
\mathbb{E}[    \mathbf{1}_{[ H(u_1) = H(v_1) ]}
    \mathbf{1}_{[ H(u_2) = H(v_2) ]}],\\
&\le
\frac{1}{D}
\sum_{\substack{u_1 \ne v_1 \\ u_2 \ne v_2}}
\mathbb{E}\!\left[
    X_{u_1} Y_{v_1}
    X_{u_2} Y_{v_2}
    S(u_1) \overline{S(v_1)}
    \overline{S(u_2)} S(v_2)
\right], \\
&\le
\frac{1}{D}
\sum_{\substack{u_1 \ne v_1 \\ u_2 \ne v_2}}
\mathbb{E}\!\left[
    |X_{u_1}|\, |Y_{v_1}|\,
    |X_{u_2}|\, |Y_{v_2}|\,
    S(u_1) \overline{S(v_1)}
    \overline{S(u_2)} S(v_2)
\right], \\
&=
\frac{1}{D}
\mathbb{E}\!\left[
\left|\left(
\sum_{u \ne v \in [d]^p}
    |X_{u}|\, |Y_{v}|\,
    S(u) \overline{S(v)}
\right)\right|^{2}
\right]. \label{eq:second moment bound for tensor sketch with bucket hash}
\end{align}
%{\color{red} Till Eq 41 I could follow the proof. }
% 
We bound the above equation using the second-moment bound of Lemma~\ref{lem:complex AMS}.
Therefore, we begin by restating the second-moment bound in the proof of Lemma~\ref{lem:complex AMS},
\begin{align}
\mathbb{E}\left[\left|\left(
\sum_{u , v \in [d]^p}
    |X_{u}|\, |Y_{v}|\,
    S(u) \overline{S(v)}
\right)\right|^{2}\right]
= \left(\langle \mathbf{x}, \mathbf{y} \rangle^{2} + \|\mathbf{x}\|_2^{2}\, \|\mathbf{y}\|_2^{2} - \sum_{i=1}^{d} x_i^{2} y_i^{2}\right)^{p}. \label{eq:exact bound tensor sketch z}
\end{align}
Now, we expand the term $\left|\left(
\sum_{u , v \in [d]^p}
    |X_{u}|\, |Y_{v}|\,
    S(u) \overline{S(v)}
\right)\right|^{2}$
from the above equation as follows, 

\begin{align}
    &\left|\left(
\sum_{u , v \in [d]^p}
    |X_{u}|\, |Y_{v}|\,
    S(u) \overline{S(v)}
\right)\right|^{2} = \left| \sum_{u  \in [d]^p}
    |X_{u}|\, |Y_{v}|  + \sum_{\substack{u, v \in [d]^p \\ u \neq v}}
    |X_{u}|\, |Y_{v}|\,
    S(u) \overline{S(v)}  \right|^{2},\\
&= \left(
\sum_{u \in [d]^p}
    |X_{u}|\, |Y_{u}|\, 
\;+\;
\sum_{\substack{u, v \in [d]^p \\ u \neq v}}
    |X_{u}|\, |Y_{v}|\,
    S(u) \overline{S(v)}  
\right)\times \cdots \notag\\
&\qquad \qquad \qquad \qquad \ \cdots \times \overline{\left(
\sum_{u \in [d]^p}
    |X_{u}|\, |Y_{u}|\,
\;+\;
\sum_{\substack{u, v \in [d]^p \\ u \neq v}}
    |X_{u}|\, |Y_{v}|\,
    S(u) \overline{S(v)}  
\right)}, \\
&= \left(
\sum_{u \in [d]^p}
    |X_{u}|\, |Y_{u}|\, +
\sum_{\substack{u, v \in [d]^p \\ u \neq v}}
    |X_{u}|\, |Y_{v}|\,
    S(u) \overline{S(v)}  
\right)\times \cdots \notag\\
&\qquad \qquad \qquad \qquad \ \cdots \times \left(
\sum_{u \in [d]^p}
    |X_{u}|\, |Y_{u}|\,
\;+\;
\sum_{\substack{u, v \in [d]^p \\ u \neq v}}
    |X_{u}|\, |Y_{v}|\,
    \overline{S(u)} S(v)  
\right).
\end{align}
% the expanded form of the expression from Lemma~\ref{lem:complex AMS} and, through a sequence
% of bounding steps, derive the final upper bound on the expectation stated in
% Equation~\eqref{eq:second moment bound for tensor sketch with bucket hash}. {\color{red} Connect with Eq 38 and 39}
By further expanding the RHS of the above equation, we get
\begin{align}
&\left|\left(
\sum_{u , v \in [d]^p}
    |X_{u}|\, |Y_{v}|\,
    S(u) \overline{S(v)}
\right)\right|^{2} =
\sum_{u_{1}, u_{2} \in [d]^p}
    |X_{u_{1}}|\, |Y_{u_{1}}|\,
    |X_{u_{2}}|\, |Y_{u_{2}}|\, + \cdots \notag
\\
&\quad
\cdots +
\sum_{u_{1} \in [d]^p}
\sum_{\substack{u_{2}, v_{2} \in [d]^p \\ u_{2} \neq v_{2}}}
    |X_{u_{1}}|\, |Y_{u_{1}}|\,
    |X_{u_{2}}|\, |Y_{v_{2}}|\,
    \overline{S(u_{2})} S(v_{2}) + \cdots  \notag
\\
&\quad
\cdots + 
\sum_{\substack{u_{1}, v_{1} \in [d]^p \\ u_{1} \neq v_{1}}}
\sum_{u_{2} \in [d]^p}
    |X_{u_{1}}|\, |Y_{v_{1}}|\,
    |X_{u_{2}}|\, |Y_{u_{2}}|\,
    S(u_{1}) \overline{S(v_{1})}\, + \cdots \notag
\\
&\quad
 \cdots + \sum_{\substack{u_{1}, v_{1} \in [d]^p \\ u_{1} \neq v_{1}}}
\sum_{\substack{u_{2}, v_{2} \in [d]^p \\ u_{2} \neq v_{2}}}
    |X_{u_{1}}|\, |Y_{v_{1}}|\,
    |X_{u_{2}}|\, |Y_{v_{2}}|\,
    S(u_{1}) \overline{S(v_{1})}\,
    \overline{S(u_{2})} S(v_{2}). \label{eq:second moment tensor sketch}
\end{align}
We know that $u_{2} \neq v_{2},\forall u_{2},v_{2} \in [d]^{p}$,
\begin{align}
    \sum_{u_{1} \in [d]^p}
\sum_{\substack{u_{2}, v_{2} \in [d]^p \\ u_{2} \neq v_{2}}}
    |X_{u_{1}}|\, |Y_{u_{1}}|\,
    |X_{u_{2}}|\, |Y_{v_{2}}|\,\mathbb{E}[
    \overline{S(u_{2})} S(v_{2})]
        = 0,
\end{align}
as   $\mathbb{E}\!\left[S(u_{2}) \overline{S(v_{2})}\right] = 0, \, \forall \, u_{2} \neq v_{2} \in [d]^p. $ Similarly, for $u_{1} \neq v_{1},\forall u_{1},v_{1} \in [d]^{p}$,
\begin{align}
    \sum_{\substack{u_{1}, v_{1} \in [d]^p \\ u_{1} \neq v_{1}}}
\sum_{u_{2} \in [d]^p}
    |X_{u_{1}}|\, |Y_{v_{1}}|\,
    |X_{u_{2}}|\, |Y_{u_{2}}|\,
    \mathbb{E}[S(u_{1}) \overline{S(v_{1})}]
        = 0,
\end{align}
as $\mathbb{E}\!\left[\overline{S(u_{1})} S(v_{1})\right] = 0,  
    \, \forall \, u_{1} \neq v_{1} \in [d]^p.$ Substituting this into Equation~\eqref{eq:second moment tensor sketch} upon computing expectation, we get
\begin{align*}
     \mathbb{E}  \Bigg|\Bigg(
\sum_{u , v \in [d]^p}
    |X_{u}|\, |Y_{v}|\,
    S(u) &\overline{S(v)}
\Bigg)\Bigg|^{2} =  \sum_{u_{1}, u_{2} \in [d]^p}
    |X_{u_{1}}|\, |Y_{u_{1}}|\,
    |X_{u_{2}}|\, |Y_{u_{2}}|\, +  \cdots  \notag
\\&  \cdots +\mathbb{E}\Bigg[\sum_{\substack{u_{1}, v_{1} \in [d]^p \\ u_{1} \neq v_{1}}}
\sum_{\substack{u_{2}, v_{2} \in [d]^p \\ u_{2} \neq v_{2}}}
    |X_{u_{1}}|\, |Y_{v_{1}}|\,
    |X_{u_{2}}|\, |Y_{v_{2}}|\ \times \cdots 
    \\&\cdots \times  S(u_{1}) \overline{S(v_{1})}\,\overline{S(u_{2})} S(v_{2})\Bigg],\\
    &=  \left\langle X,Y\right\rangle^{2} +\mathbb{E}  \left|\left(
\sum_{u \neq v }
    |X_{u}|\, |Y_{v}|\,
    S(u) \overline{S(v)}
\right)\right|^{2}.
\end{align*}
Now, we conclude that,
 \begin{align}
\mathbb{E}  \left|\left(
\sum_{u \neq v }
    |X_{u}|\, |Y_{v}|\,
    S(u) \overline{S(v)}
\right)\right|^{2}      &=  \mathbb{E}  \left|\left(
\sum_{u , v \in [d]^p}
    |X_{u}|\, |Y_{v}|\,
    S(u) \overline{S(v)}
\right)\right|^{2}  - \left\langle \mathbf{x,y}\right\rangle^{2p}.
\end{align}
Now substitute the value of $\mathbb{E}\left[\left|\left(
\sum_{u , v \in [d]^p}
    |X_{u}|\, |Y_{v}|\,
    S(u) \overline{S(v)}
\right)\right|^{2}\right]$ from Equation~\eqref{eq:exact bound tensor sketch z}, we get  
\begin{align}
\mathbb{E}  \left|\left(
\sum_{u \neq v }
    |X_{u}|\, |Y_{v}|\,
    S(u) \overline{S(v)}
\right)\right|^{2}      &=  \left(\langle \mathbf{x}, \mathbf{y} \rangle^{2} + \|\mathbf{x}\|_2^{2}\, \|\mathbf{y}\|_2^{2}
- \sum_{i=1}^{d} x_i^{2} y_i^{2}\right)^{p}  - \left\langle \mathbf{x,y}\right\rangle^{2p}. \label{eq:exact bound for tensor sketch non diagonal terms}
\end{align}
Further we substitute this value in Equation~\eqref{eq:second moment bound for tensor sketch with bucket hash}, we get
\begin{align}
    &\mathbb{E}\!\left[
\left|\left(
\sum_{u \ne v}
    X_{u} Y_{v} 
    S(u) \overline{S(v)} 
    \mathbf{1}_{[ H(u) = H(v) ]}
\right)\right|^{2}
\right] \notag\\
&\qquad\leq \frac{1}{D}\left(\left(\langle \mathbf{x}, \mathbf{y} \rangle^{2} + \|\mathbf{x}\|_2^{2}\, \|\mathbf{y}\|_2^{2}
- \sum_{i=1}^{d} x_i^{2} y_i^{2}\right)^{p}  - \left\langle \mathbf{x,y}\right\rangle^{2p}\right).
\end{align}
Now we can compute the second moment using Equation~\eqref{eq: second moment for tensor sketch to be used for variance} as follows,
\begin{align}
  \mathbb{E}\left[  | \hat{k}_{C}(\mathbf{x}, \mathbf{y})|^{2}\right] &= \left\langle X,Y \right\rangle^{2} +    \mathbb{E}\!\left[
\left|\left(
\sum_{u \ne v}
    X_{u} Y_{v} 
    S(u) \overline{S(v)} 
    \mathbf{1}_{[ H(u) = H(v) ]}
\right)\right|^{2}
\right],\\
&\leq \left\langle \mathbf{x},\mathbf{y} \right\rangle^{2p} + \frac{1}{D}\left(\left(\langle \mathbf{x}, \mathbf{y} \rangle^{2} + \|\mathbf{x}\|_2^{2}\, \|\mathbf{y}\|_2^{2}
- \sum_{i=1}^{d} x_i^{2} y_i^{2}\right)^{p}  - \left\langle \mathbf{x,y}\right\rangle^{2p}\right).
\end{align}
Now, compute variance as follows
\begin{align}
    \operatorname{Var}(\hat{k}_{C}(\mathbf{x}, \mathbf{y})) &= \mathbb{E}\left[  | \hat{k}_{C}(\mathbf{x}, \mathbf{y})|^{2}\right] - \left|\mathbb{E}\left[   \hat{k}_{C}(\mathbf{x}, \mathbf{y})\right]\right|^{2}, \\
    &\leq \left\langle \mathbf{x},\mathbf{y} \right\rangle^{2p} + \frac{1}{D}\left(\left(\langle \mathbf{x}, \mathbf{y} \rangle^{2} + \|\mathbf{x}\|_2^{2}\, \|\mathbf{y}\|_2^{2}
- \sum_{i=1}^{d} x_i^{2} y_i^{2}\right)^{p}  - \left\langle \mathbf{x,y}\right\rangle^{2p}\right) - \langle \mathbf{x}, \mathbf{y} \rangle^{2p},\\
    &= \frac{1}{D}\left(\left(\langle \mathbf{x}, \mathbf{y} \rangle^{2} + \|\mathbf{x}\|_2^{2}\, \|\mathbf{y}\|_2^{2} - \sum_{i=1}^{d} x_i^{2} y_i^{2}\right)^{p}  - \left\langle \mathbf{x,y}\right\rangle^{2p}\right). \label{eq:exact variance bound for tensor sketch}
\end{align}
We can upper bound the above equation by using inequality $\langle\mathbf{x,y}\rangle^{2} \leq \|\mathbf{x}\|^{2}_{2}\|\mathbf{y}\|^{2}_{2}$, then
\begin{align}
   \operatorname{Var}(\hat{k}_{C}(\mathbf{x}, \mathbf{y})) &\leq  \frac{1}{D}\left(\left( 2\|\mathbf{x}\|_2^{2}\, \|\mathbf{y}\|_2^{2} \right)^{p}  - \|\mathbf{x}\|_2^{2p}\, \|\mathbf{y}\|_2^{2p}\right),\\
   &= \frac{\left( 2^{p}-1 \right)}{D}\|\mathbf{x}\|_2^{2p}\, \|\mathbf{y}\|_2^{2p}.
\end{align}

\end{proof}
\begin{remark}[Sketching time for Complex \texttt{TensorSketch}]
Let $\mathbf{x} \in \mathbb{R}^{d}$ and let $p \ge 1$ be an integer.
The Complex \texttt{TensorSketch} of $\mathbf{x}^{\otimes p}$ with sketch dimension $D$
can be computed in $O(p (\mathrm{nnz}(\mathbf{x}) + D \log D))$ time. This follows because \texttt{TensorSketch} avoids explicitly forming the
tensor $\mathbf{x}^{\otimes p}$.
Instead, it applies $p$ independent Complex \texttt{CountSketch} to $\mathbf{x}$,
each takes $O(\operatorname{nnz}(\mathbf{x}))$ time, and combines the resulting $p$
sketches using circular convolution, which is implemented via FFT in
$O(pD \log D)$ time.
\end{remark}

\begin{lemma}
\label{lem:complex AMS}
Let $\mathbf{x}, \mathbf{y} \in \mathbb{R}^{d}$, let $p > 1$ be an integer, and let 
$s_{1}, \ldots, s_{p} : [d] \to \{1, \omega, \omega^{2}, \omega^{3}\}$ 
be independent functions, each taking values uniformly from the four fourth roots of unity.  
Define
\begin{align}
    Z \;=\; \prod_{j=1}^{p} Z_{s_j}(\mathbf{x}) \, \overline{Z_{s_j}(\mathbf{y})},
\end{align}
Where
\begin{align}
    Z_{s_j}(\mathbf{x}) &= \sum_{i=1}^{d} x_i\, s_j(i), 
    &
    Z_{s_j}(\mathbf{y}) &= \sum_{i=1}^{d} y_i\, s_j(i).
\end{align}
Then,
\begin{align}
    \mathbb{E}[Z] &= \langle \mathbf{x}, \mathbf{y} \rangle^{p}, \\
    \mathrm{Var}[Z] 
        &= \left( \langle \mathbf{x}, \mathbf{y} \rangle^{2} 
           + \|\mathbf{x}\|_{2}^{2}\, \|\mathbf{y}\|_{2}^{2}
           - \sum_{i=1}^{d} x_i^{2} y_i^{2} \right)^{p}
           \;-\; \langle \mathbf{x}, \mathbf{y} \rangle^{2p}, \\
        &\leq 2^{p}\, \|\mathbf{x}\|_{2}^{2p}\, \|\mathbf{y}\|_{2}^{2p}.
\end{align}
\end{lemma}

\begin{proof}
Following the approach of~\cite{braverman2010ams4wiseindependenceproduct}, adapted from~\cite{pham2013fast}[Lemma 8], we compute the expectation and variance of $Z$. First, we consider the expectation. For each $j$, we note that
\begin{align}
\mathbb{E}\!\left[ Z_{s_j}(\mathbf{x})\, \overline{Z_{s_j}(\mathbf{y})} \right]
&= 
\mathbb{E}\!\left[
\left( \sum_{i=1}^{d} x_i\, s_j(i) \right)
\left( \sum_{k=1}^{d} y_k\, \overline{s_j(k)} \right)
\right], \\
&= \sum_{i=1}^{d}\sum_{k=1}^{d} x_i y_k\, \mathbb{E}[s_j(i)\overline{s_j(k)}], \\
&= \sum_{i=1}^{d} x_i y_i\, \mathbb{E}[|s_j(i)|^2] +  \sum_{i\neq k} x_i y_k\, \mathbb{E}[s_j(i)\overline{s_j(k)}], \\ 
&= \langle \mathbf{x}, \mathbf{y} \rangle,
\end{align}
Where, $\mathbb{E}[s_j(i)\overline{s_j}(k)] = 0 , \forall i\neq k$ and $\mathbb{E}[|s_j(i)|^2] = 1,  \forall\  i\in [d]$.

\noindent Since the functions $s_j$ are independent across different $j$, we have
\begin{align}
\mathbb{E}[Z] 
= \prod_{j=1}^{p} \mathbb{E}[Z_{s_j}(\mathbf{x}) \overline{Z_{s_j}(\mathbf{y})}]
= \langle \mathbf{x}, \mathbf{y} \rangle^{p}.
\end{align}

\noindent Next, to bound the variance,
\begin{align}
\mathrm{Var}(Z) = \mathbb{E}[|Z|^{2}] - |(\mathbb{E}[Z])|^{2}.
\end{align}

\noindent Because functions is independent across different $j$, we may write
\begin{align}\label{complex second moment ams}
\mathbb{E}[|Z|^{2}]
= \prod_{j=1}^{p} \mathbb{E}\!\left[ |\left( Z_{s_j}(\mathbf{x})\, \overline{Z_{s_j}(\mathbf{y})} \right)|^{2} \right].
\end{align}

\noindent For each $j$, expanding the square gives
\begin{align}
&\mathbb{E}\!\left[ |\left( Z_{s_j}(\mathbf{x})\, \overline{Z_{s_j}(\mathbf{y})} \right)|^{2} \right] \notag\\
&=
\mathbb{E}\!\left[
\left( \sum_{i=1}^{d} x_i\, s_j(i) \right)
\left( \sum_{k=1}^{d} y_k\, \overline{s_j(k)} \right)\left( \sum_{i=1}^{d} x_i\, \overline{s_j(i)} \right)
\left( \sum_{k=1}^{d} y_k\, s_j(k) \right)
\right] ,\\
&=
\sum_{i=1}^{d} \sum_{i'=1}^{d} \sum_{k=1}^{d} \sum_{k'=1}^{d}
x_i\, x_{i'}\, y_k\, y_{k'}\;
\mathbb{E}\!\left[ s_j(i)\overline{s_j(k)}\overline{s_j(i')}s_j(k') \right].
\end{align}

\noindent Observing that 
$\mathbb{E}[s_j(i)\overline{s_j(k)}\overline{s_j(i')}s_j(k')]$ 
is nonzero only when the indices form pairs (including the possibility that all four are identical), we have
\begin{align}
\mathbb{E}[s_j(i)\overline{s_j(k)}\overline{s_j(i')}s_j(k')]
=
\begin{cases}
1, & \text{if } i = k = i' = k', \\[4pt]
1, & \text{if } i = k \ne i' = k', \\[4pt]
1, & \text{if } i = i' \ne k = k', \\[4pt]
0, & \text{otherwise}.
\end{cases}
\end{align}
\noindent The contribution from terms with $i = k = i' = k'$ is 
\begin{align}
\sum_{i=1}^{d} x_i^{2} y_i^{2}.
\end{align}
\noindent Terms with $i = k \ne i' = k'$ contribute
\begin{align}
\sum_{i \ne i'} x_i y_i\, x_{i'} y_{i'}
= \left( \sum_{i=1}^{d} x_i y_i \right)^{2}
 - \sum_{i=1}^{d} x_i^{2} y_i^{2}
= \langle \mathbf{x}, \mathbf{y} \rangle^{2} - \sum_{i=1}^{d} x_i^{2} y_i^{2}.
\end{align}
Finally, for $i = i' \ne k = k'$ we obtain
\begin{align}
\sum_{i \ne k} x_i^{2} y_k^{2}
= \|\mathbf{x}\|_2^{2}\, \|\mathbf{y}\|_2^{2}
 - \sum_{i=1}^{d} x_i^{2} y_i^{2}.
\end{align}
Thus, summing these contributions, we have
\begin{align}
\mathbb{E}\!\left[| \left( Z_{s_j}(\mathbf{x})\, Z_{s_j}(\mathbf{y}) \right) |^{2} \right]
&=
\sum_{i=1}^{d} x_i^{2} y_i^{2}
+ \left( \langle \mathbf{x}, \mathbf{y} \rangle^{2} + \|\mathbf{x}\|_2^{2}\, \|\mathbf{y}\|_2^{2}
 - 2\sum_{i=1}^{d} x_i^{2} y_i^{2} \right),
\\[4pt]
&=
 \langle \mathbf{x}, \mathbf{y} \rangle^{2} + \|\mathbf{x}\|_2^{2}\, \|\mathbf{y}\|_2^{2}
- \sum_{i=1}^{d} x_i^{2} y_i^{2}.
\end{align}
Substituting this bound into Equation~\eqref{complex second moment ams} yields
\begin{align}    
\mathbb{E}[|Z|^{2}]
= \left(\langle \mathbf{x}, \mathbf{y} \rangle^{2} + \|\mathbf{x}\|_2^{2}\, \|\mathbf{y}\|_2^{2}
- \sum_{i=1}^{d} x_i^{2} y_i^{2}\right)^{p}, \label{eq:second moment x^p}
\end{align}
Which completes the proof since
\begin{align}
\mathrm{Var}(Z)
&= \mathbb{E}[|Z|^{2}] - \langle \mathbf{x}, \mathbf{y} \rangle^{2p},\\
&= \left( \langle \mathbf{x}, \mathbf{y} \rangle^{2} + \|\mathbf{x}\|_2^{2}\, \|\mathbf{y}\|_2^{2}
- \sum_{i=1}^{d} x_i^{2} y_i^{2}\right)^{p} - \langle \mathbf{x}, \mathbf{y} \rangle^{2p} \label{eq:variance quality tensor kronecker p times}.
\end{align}
Using the Cauchy--Schwarz inequality, 
$\langle \mathbf{x}, \mathbf{y} \rangle^{2} \le \|\mathbf{x}\|_2^{2}\, \|\mathbf{y}\|_2^{2}$,  
and noting that 
$\sum_{i=1}^{d} x_i^{2} y_i^{2} \ge 0$,  
it follows that
\begin{align}
\mathrm{Var}(Z) \leq  2^{p} \|\mathbf{x}\|_2^{2p}\|\mathbf{y}\|_2^{2p}.
\end{align}
\medskip

\end{proof}